\def\year{2025}\relax
\documentclass[letterpaper]{article} 
\usepackage[]{aaai25}  
\usepackage{times}  
\usepackage{helvet}  
\usepackage{courier}  
\usepackage[hyphens]{url}  
\usepackage{graphicx} 
\usepackage{natbib}  
\usepackage{caption} 
\DeclareCaptionStyle{ruled}{labelfont=normalfont,labelsep=colon,strut=off} 
\usepackage{algorithm}
\usepackage{algpseudocode}
\usepackage{newfloat}
\usepackage{listings}
\floatstyle{ruled}
\newfloat{listing}{tb}{lst}{}
\floatname{listing}{Listing}
\usepackage{booktabs}
\usepackage{amsfonts}       
\usepackage{nicefrac}       
\usepackage{microtype}      
\usepackage{xcolor}         

\usepackage{amsmath}       
\usepackage{mathtools}
\usepackage{amsthm}
\usepackage{amssymb}
\usepackage{ifthen}
\usepackage{bm}

\usepackage{multirow}
\usepackage{subcaption}
\usepackage{csquotes}
\usepackage{makecell}

\title{General Uncertainty Estimation with Delta Variances}
\author {
    Simon Schmitt, \textsuperscript{\rm 1,2}
    John Shawe-Taylor, \textsuperscript{\rm 2}
    Hado van Hasselt\textsuperscript{\rm 1}
}
\affiliations {
    \textsuperscript{\rm 1} DeepMind \\
    \textsuperscript{\rm 2} University College London, UK \\
    suschmitt@google.com
}

\newtheorem*{theorem*}{Theorem}
\newtheorem{definition}{Definition}
\newtheorem{proposition}{Proposition}
\newtheorem*{proposition*}{Proposition}

\newtheorem{lemma}{Lemma}
\newtheorem*{lemma*}{Lemma}

\newcommand{\defeq}{\vcentcolon=}

\newcommand{\trans}[1]{{#1}^\top}
\newcommand{\Expectation}[2][]{\mathbb{E}_{#1}\left[#2\right]}
\newcommand{\Variance}[2][]{\mathbb{V}_{#1}\left[#2\right]}

\newcommand{\gradtheta}{\nabla_\theta}

\newcommand{\Data}{\mathcal{D}}

\newcommand{\thetabar}{{\bar{\theta}}}

\newcommand{\thetatrue}{\theta_\mathrm{True}}

\newcommand{\normal}[1]{\mathcal{N}(#1)}

\newcommand{\bigoh}{O}
\newcommand{\Bigoh}[1]{\bigoh\left( #1 \right)}

\newcommand{\derivedQ}{u}
\newcommand{\derivedQtheta}{\derivedQ_\theta}
\newcommand{\derivedQthetabar}{\derivedQ_\thetabar}
\newcommand{\quadForm}[2][]{#2^\top \, #1 \, #2}

\newcommand{\deltaThetaSample}{z}

\newcommand{\ftheta}{f_\theta}
\newcommand{\fthetax}{\ftheta(x)}
\newcommand{\deltaUz}{\Delta_{\derivedQ(z)}}
\newcommand{\deltaUx}{\Delta_{\derivedQ(x)}}
\newcommand{\deltaF}{\Delta_{f(x)}}

\newcommand{\deltatheta}{\Delta\theta}

\newcommand{\iinN}{i\sim U(1, \ldots, N)}

\newcommand{\PosteriorCovarianceMatrix}{\Sigma}
\newcommand{\iHessian}{H^{-1}_f}
\newcommand{\EigenvalueExampleMatrix}{A}
\newcommand{\fixedpointSelfFunction}{F}
\newcommand{\fixpointZeroFunction}{G}
\newcommand{\ziterate}{w}
\newcommand{\zfixed}{\ziterate^{*}}

\newcommand{\identity}{I}

\newcommand{\pfthetaxi}{p(y_i| \ftheta(x_i))}
\newcommand{\pfthetax}{p(y| \fthetax)}
\newcommand{\bootRademacher}{r}
\newcommand{\thetaRademacher}{\theta_{\bootRademacher}}
\newcommand{\DVariance}[2][]{\mathbb{DV}_{#1}\left[#2\right]}

\nocopyright
\begin{document}

\maketitle

\begin{abstract}
Decision makers may suffer from uncertainty induced by limited data. This may be mitigated by accounting for epistemic uncertainty, which is however challenging to estimate efficiently for large neural networks.

To this extent we investigate Delta Variances, a family of algorithms for epistemic uncertainty quantification, that is computationally efficient and convenient to implement. It can be applied to neural networks and more general functions composed of neural networks. As an example we consider a weather simulator with a neural-network-based step function inside -- here Delta Variances empirically obtain competitive results at the cost of a single gradient computation.

The approach is convenient as it requires no changes to the neural network architecture or training procedure. We discuss multiple ways to derive Delta Variances theoretically noting that special cases recover popular techniques and present a unified perspective on multiple related methods. Finally we observe that this general perspective gives rise to a natural extension and empirically show its benefit.

\end{abstract}

\noindent

\section{Introduction}

Decision makers often need to act given limited data. Accounting for the resulting uncertainty (epistemic uncertainty) may be helpful for active learning~\cite{MacKay:92c}, exploration \cite{Duff:2002,Auer:2002} and safety \cite{Heger:94}.

How to measure epistemic uncertainty efficiently for large neural networks is active research. Computational efficiency is important because even a single evaluation (e.g. a forward pass through a neural network) can be expensive. Popular approaches compute an ensemble of predictions using bootstrapping or MC dropout and incur a multiplicative computational overhead. Other approaches are faster but require changes to the predictors architecture and training procedure~\cite{Amersfoort:2020UncertaintyDUQ}. 

In this paper we propose the \emph{Delta Variance} family of algorithms which connects and extends Bayesian, frequentist and heuristic notions of variance. 
Delta Variances require no changes to the architecture or training procedure while incurring the cost of little more than a gradient computation. We present further appealing properties and benefits in Table~\ref{tab:delta_benefits}.

The approach can be applied to neural networks or functions that contain neural networks as building blocks to compute a quantity of interest. For instance we could learn a step-by-step dynamics model and then use it to infer some utility function for decision making.

\begin{figure}[ht]
    \centering
    \includegraphics[width=.47\textwidth]{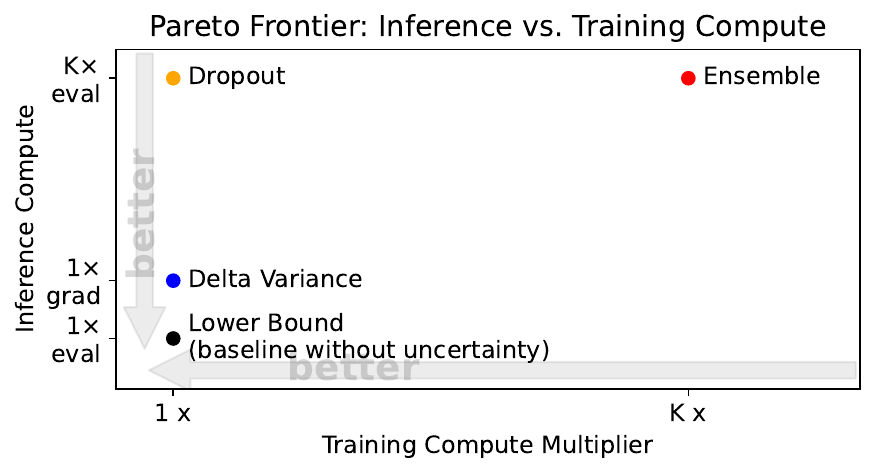}
    \caption{We compare the computational overhead of training and evaluating different variance estimators. Delta Variances are favourable in terms of computational efficiency. They incur negligible training overhead while inference  incurs the cost of a regular gradient pass making them more efficient than the alternatives considered.
    Monte-Carlo Dropout also incurs negligible training overhead, but requires $K$ independent evaluations for inference.
    Most expensive are Bootstrapped Ensembles requiring $K\times$ repeated computations.
    }
    \label{fig:delta_pareto_inference_vs_training}
\end{figure}

\begin{table*}
\centering
\begin{tabular}{ c | c c c }
 & \textbf{Delta Variances} & \textbf{Ensemble} & \textbf{MC-Dropout} \\
\toprule
\multicolumn{4}{l}{\textbf{Efficiency}} \\						
Inference cost	&	$\mathbf{1 \times gradient}$	&	$K \times\ \mathrm{evaluations}$	&	$K \times \ \mathrm{evaluations}$	\\
Training overhead	&	$\mathbf{1 \times}$	&	$K \times$	&	$\mathbf{1 \times}$	\\
Memory overhead	&	$2 \times$	&	$K \times$	&	$\mathbf{1 \times}$	\\
\midrule
\multicolumn{4}{l}{\textbf{Ease of Use}} \\						
No architecture requirements	&	\textbf{\checkmark}	&	\textbf{\checkmark}	&		\\
No change to training procedure	&	\textbf{\checkmark}	&		&		\\
Deterministic result	&	\textbf{\checkmark}	&	\textbf{\checkmark}	&		\\
\bottomrule
\end{tabular}
\caption{\label{tab:delta_benefits} 
Delta Variances have appealing benefits: The prototypical variant with diagonal $\PosteriorCovarianceMatrix$ is computationally efficient requiring only a gradient pass for inference while other methods evaluate the neural network multiple times. Furthermore they easily build on the existing training procedure and can even be added post hoc after training. Delta Variances do not require architecture changes such as introducing dropout layers or training procedure changes as needed for ensembling or even hyper parameter search. Finally Delta Variances have a simple closed form expression that yields reproducible deterministic results.}
\end{table*}

In Section~\ref{sec:delta_experiments} we consider the GraphCast weather forecasting system with a neural network step function \cite{lam:2023GraphCast}. We then compute the epistemic uncertainty of various derived quantities such as the expected precipitation or wind-turbine-power at a particular location. 

Section~\ref{sec:delta_analysis} observes how instances of the Delta Variance family can be derived using different assumptions and theoretical frameworks. We begin with a Bernstein-von Mises plus Delta Method derivation, which relies on strong assumptions. We conclude with an influence function based derivation, which relies only on mild assumptions. Interestingly the resulting instances are not only similar, but also become identical as the number of observed data-points grows. 

Formalizing the Delta Variance family allows us to connect Bayesian and frequentist notions of variance, adversarial robustness and anomaly detection in a unified perspective. This perspective can be used to answer questions such as: What happens if we use a Bayesian variance, but our neural network does not meet all theoretical assumptions -- what is a theoretically sound interpretation of the number that we compute? 
To further highlight the generality of this unified perspective we propose a novel Delta Variance in Section~\ref{sec:delta_learning_sigma} for which we observe empirically improvements in Section~\ref{sec:delta_experiments}.

When applied to the state-of-the-art GraphCast weather forecasting system we observe favourable results. In comparison to popular related approaches such as ensemble methods our method exhibits similar quality while requiring fewer computational resources.

\paragraph{What is Epistemic Variance?}
Given limited data the parameters $\theta$ of a parametric function e.g. a neural network $f_\theta$ can only be identified with limited certainty. The resulting parameter uncertainty translates into uncertainty in the outputs of $f_\theta(x)$ and any other function $\derivedQtheta(z)$ that depends on $\theta$. We define Epistemic Variance as the output variance induced by the posterior distribution over parameters given the function $f$ and its training data $p(\theta|f, \Data)$:
$$ \Variance[\theta \sim p(\theta|f, \Data)]{f_\theta(x)} $$
Section~\ref{sec:delta_problem_description} extends this definition to any function $\derivedQtheta$ that depends on parameters $\theta$ that were estimated using $f_\theta$ and $\Data$. As an illustration consider Section~\ref{sec:delta_experiments} where $\ftheta$ is a learned weather dynamics model and $\derivedQtheta$ are weather dependent utility functions - e.g. a wind turbine power yield forecast. The epistemic variance of $\derivedQtheta(z)$ is then $\Variance[\theta \sim p(\theta|f, \Data)]{\derivedQtheta(z))}$.

\paragraph{What is a Delta Variance Estimator?}
Variance estimators of the Delta Variance family are all of the following parametric form:
$$
\quadForm[\PosteriorCovarianceMatrix]{\deltaF} $$
being a vector-matrix-vector product using the gradient vector of $f$: $\deltaF \defeq \gradtheta f_\theta(x)$ while leaving some flexibility in the choice of matrix $\PosteriorCovarianceMatrix$.
In Section~\ref{sec:delta_analysis} and Table~\ref{tab:delta_interpretations} we discuss different choices of $\PosteriorCovarianceMatrix$ and their properties.
For $\PosteriorCovarianceMatrix_f=\frac{1}{N}F_f^{-1}$ being the inverse Fisher information matrix divided by the number of data-points $N$, it can be shown under suitable conditions that 
$$ \Variance[\theta \sim p(\theta|f, \Data)]{f_\theta(x)} \approx \quadForm[\PosteriorCovarianceMatrix_f]{\deltaF} $$
Its name is inspired by the closely related Delta Method \cite{Lambert:1765Beytraege,gauss:1823ErrorTheory,Doob:1935LimitingDistributions}
-- see \citet{Gorroochurn:2020WhoInventedDelta} for a historic account -- which provides one of many ways to derive Delta Variances.

\paragraph{What is a Quantity of Interest?}
Sometimes we use a neural network $\ftheta$ to learn predictions, that are then used to compute a downstream quantity of interest.
For instance we could learn a step-by-step weather dynamics model $\fthetax$ and then use it to infer some utility function for decision making $\derivedQtheta(z)$.
Given limited training data neither the neural network’s prediction nor the downstream quantity of interest will be exact.
The definition of Epistemic Variance and the Delta Variance family both extend conveniently to quantities of interest $\derivedQtheta$. In fact we only need to replace $\deltaF$ by $\deltaUz\defeq \gradtheta \derivedQtheta(z)$:
$$ \Variance[\theta \sim p(\theta|f, \Data)]{\derivedQtheta(z)} \approx \quadForm[\PosteriorCovarianceMatrix_f]{\deltaUz}$$
Conveniently we can still use the same $\PosteriorCovarianceMatrix_f$ as before.
It is independent of $\derivedQtheta$ and can be re-used for various quantities of interest.

\section{Notation}
We consider a function $\fthetax$ with parameters $\theta$ trained on a dataset $\Data$ of size $N$. We strive to estimate the uncertainty introduced by training on limited data. To admit Bayesian interpretations we assume that the function $\fthetax$ corresponds to a density model or probability mass function $p(y| \fthetax)$. Intuitively $p(y| \fthetax)$ is a simple pre-defined transformation converting the outputs of $\fthetax$ into a probability distribution.
For example $\fthetax$ may define the mean of a normal distribution:
$\pfthetax = \frac{1}{\sqrt{2\pi\sigma^2}}\exp\left(-\frac{(y-\fthetax)^2}{2\sigma^2}\right)$
Alternatively it may represent the logits for a categorical distribution.
In our examples -- unless specified otherwise -- $\pfthetax$ will represent a normal distribution with mean $\fthetax$ and constant variance or a multi-variate with constant diagonal variance.
Note that the probabilistic interpretation can also be implicit e.g. when $\fthetax$ is trained with a negative-log-likelihood-equivalent loss, such as $L_2$ regression or cross-entropy).
Unless otherwise specified we assume that $\theta$ has been trained until convergence -- i.e. equals the (local) maximum likelihood estimate $\thetabar$. Under appropriate conditions $\thetabar$ converges to the true distribution parameters $\thetatrue$.
When adopting a Bayesian view with prior belief $p(\theta)$ the posterior over parameters is defined as $p(\theta|\Data)\propto p(\Data|\theta) p(\theta)$. 
$\mathbb{E}_{z\sim p(z)}$ refers to the expectation with respect to random variable $z$ with distribution $p(z)$ -- which we shorten to  $\mathbb{E}_{z}$ or $\mathbb{E}$ when the distribution is clear. Similarly let $\Variance[]{X}\defeq\Expectation{X^2} - \Expectation[]{X}^2$.
Let 
$
F_f\defeq \mathbb{E}_{x}\mathbb{E}_{y\sim f_{\thetatrue}(x)} \gradtheta \log \pfthetax \trans{\gradtheta \log\pfthetax}|_{\theta=\thetatrue}
$
be the Fisher information matrix and let $ \hat{F}_{f} \defeq \frac{1}{N} \sum_{(x_i, y_i) \in \Data} \quadForm[]{\gradtheta \log \pfthetaxi} |_{\theta=\thetabar}$ be the empirical Fisher information. For a Gaussian model the gradient simplifies to $\gradtheta \log \pfthetaxi=\frac{y-\fthetax}{\sigma^2} \gradtheta \fthetax$. Note that $\hat{F}_{f}$ is an average and does not grow with N. A related matrix is the Hessian of the loss. While the Hessian of the summed negative log-likelihood grows as $\Bigoh{N}$ we work with the average Hessian to unify the semantics: Let $H_f$ be the Hessian of the average training loss (e.g. negative log-likelihood) of all $\Data$ evaluated at $\thetabar$. This corresponds to weighing each data-point with $\frac{1}{N}$. Hence both $\hat{F}_{f}$ and $H_f$ are averages and do not grow with $N$. This allows us to make any dependences on $N$ explicit.
When strong conditions are met \citep[see][for details]{vandervaart1998:asymptotic} the Bernstein-von Mises theorem ensures that the Bayesian posterior converges to the Gaussian distribution $\normal{\thetabar,\frac{1}{N}F_f^{-1}}$ in total variation norm independently of the choice of prior as the number of data-points $N$ increases.
In Definition~\ref{def:quantity_of_interest} we consider \emph{Quantities of Interest} that we denote $\derivedQtheta$. In practice $\derivedQtheta(z)$ may be a utility function that depends on some context provided by $z$.
For a simpler exposition but without loss of generality we assume that $\derivedQtheta(z)$ is scalar valued. 
We require $\ftheta$ and $\derivedQtheta$ to have bounded second derivatives wrt. $\theta$ in order to perform first order Taylor expansions. To simplify notations we assume that $\theta$ is evaluated at the learned parameters $\thetabar$ unless specified otherwise: in particular we write $\gradtheta f_\theta(x)$ in place of $\gradtheta f_\theta(x)|_{\theta=\thetabar}$ and $\deltaF \defeq \gradtheta f_\theta(x)|_{\theta=\thetabar}$.

\section{Epistemic Variance of Quantities of Interest}
\label{sec:delta_problem_description}
Sometimes we use a neural network $\ftheta$ to learn predictions, that are then used to compute a downstream quantity of interest $\derivedQtheta$ -- see motivational examples below. Given limited training data neither the neural network’s prediction nor the downstream quantity of interest will be exact. This motivates our research question:

\begin{center}
   \textbf{
   If we estimate the parameters $\mathbf{\theta}$ of $\mathbf{\derivedQtheta}$ by learning $\mathbf{\fthetax}$, how can we quantify the epistemic uncertainty of $\mathbf{\derivedQtheta(z)}$?}
\end{center}

For a simpler exposition and without loss of generality we assume that $\derivedQtheta(z)$ predicts scalar quantities. The prototypical example is a utility function that depends on some context provided by $z$ and internally uses $\ftheta$ to compute a utility value. The derivations carry over naturally to the multi-variate case. Note that $\ftheta$ and $\derivedQtheta$ may have different input spaces. Our research focuses on the general case where $\ftheta\neq\derivedQtheta$ which has received little attention. This naturally includes the case where $\ftheta=\derivedQtheta$. 

\begin{definition}
\label{def:quantity_of_interest}
We call the real-valued function $\derivedQtheta(z)$ \emph{quantity of interest} if it depends on the same parameters $\theta$ as a related parametric model $\fthetax$.
\end{definition}

\subsection{Motivational Examples}
We consider three motivational examples for training on $\ftheta$ but evaluating a different quantity of interest $\derivedQtheta$. We will see that training $\ftheta$ is straightforward while training a predictor for $\derivedQtheta(z)$ is inefficient, impractical, or even impossible.

\begin{enumerate}
    \item 
As a simple motivation let us consider estimating the 10-year survival chance using a neural network predictor $\ftheta(x)$ of 1-year outcomes given patient features $x$:
$$ \derivedQtheta(x) = \ftheta(x)^{10}$$
This example illustrates that it may be \textbf{impossible to train $\mathbf{\derivedQtheta(x)}$ directly} unless we collect data for 9 more years, hence we train $\ftheta(x)$ and evaluate $\derivedQtheta(z)$.
\item Distinct input spaces: $\derivedQtheta(z)$ might aggregate predictions of $\fthetax$ for sets $z=\{x_1, \dots, x_k\}$: E.g. the survival chance of everyone in set of patients $z$ via $\derivedQtheta(z)\defeq \prod_{x_i\in z} \ftheta(x_i)$, or the average value of some basket of items $z$, or the chance of any advertisement from a presented set being clicked. Here training $\ftheta$ may be more convenient than training $\derivedQtheta$.

\item Multiple derived quantities: In Section~\ref{sec:delta_experiments} we compute multiple quantities of interest using the GraphCast weather forecasting system \cite{lam:2023GraphCast}.
Training a separate $\derivedQtheta$ for each of them would be cumbersome and expensive.
\end{enumerate}

\subsection{Epistemic Variance}
Here we define Epistemic Variance first from a Bayesian and then from a frequentist perspective. This allows us to formalize and quantify how parameter uncertainty from training $\ftheta$ translates to uncertainty of any quantity of interest $\derivedQtheta(z)$ that also depends on $\theta$.

\paragraph{Bayesian Definition}
Epistemic uncertainty can be formalized with a Bayesian posterior distribution over parameters given training data: $p(\theta|\Data)$. The Epistemic Variance of a function evaluation $\derivedQtheta(z)$ is then defined to be the variance induced by the posterior over $\theta$:

\begin{definition}
\label{def:epistemic_variance}
Given any function $\derivedQtheta$ and a posterior over parameters $p(\theta|f, \Data)$ resulting from training $\ftheta$ on data $\Data$ the \emph{Epistemic Variance of $\derivedQtheta(z)$} is defined as
$$ 
\Variance[\theta\sim p(\theta|f, \Data)]{\derivedQtheta(z)}
 $$
where $\Variance{X} \defeq \Expectation{X^2} - \Expectation{X}^2$.
\end{definition}

\paragraph{Frequentist Definitions}
Leave-one-out cross-validation \cite{Quenouille:1949Correlation} is a frequentist counterpart to Epistemic Variance. It computes the variance of $\derivedQtheta(z)$ induced by removing a random element from the training data and re-estimating the parameters $\theta$. It is multiplied by $N$ to match the scale of the Bayesian and Bootstrapped Epistemic Variance.

\begin{definition}
Let $\theta_{\setminus i}$ be the leave-one-out parameters resulting from training $\ftheta$ on data $\Data \setminus \{(x_i, y_i)\}$, then the \emph{Leave-one-out Variance of $\derivedQtheta(z)$} is defined as
$$ \Variance[\theta \sim LOO]{\derivedQtheta(z)} \defeq N\, \Variance[\iinN]{\derivedQ_{\theta_{\setminus i}}(z)} $$
where $U(1, \ldots, N)$ is the uniform distribution over indices.
\end{definition}

\begin{definition}
Let $\theta_{b}$ be be the parameters resulting from training $\ftheta$ on a random $\frac{1}{2}$ subset of data $\Data$, then the \emph{Bootstrap Variance of $\derivedQtheta(z)$} is defined as 
$$ \Variance[\theta \sim Bootstrap]{\derivedQtheta(z)} \defeq \frac{1}{2} \Variance[\theta_b]{\derivedQ_{\theta_b}(z)} $$
\end{definition}

\section{Delta Variance Approximators}
Delta Variance estimators are a family of efficient and convenient approximators of epistemic uncertainty. They can be used to compute the Epistemic Variance of a quantity of interest $\derivedQtheta(z)$ where the parameters $\theta$ are obtained by learning $\ftheta$ with limited data.
Given any quantity of interest $\derivedQtheta(z)$ they approximate both the Bayesian Epistemic Variance as well as the frequentist analogues: leave-one-out and bootstrapped variance.
\begin{equation*}
\underbrace{\Variance[\theta \sim p(\theta|f, \Data)]{\derivedQtheta(z)}}_{\text{Epistemic Variance}} 
\approx 
\underbrace{\quadForm[\PosteriorCovarianceMatrix]{\deltaUz}}_{\text{Delta Variance}} 
\approx \underbrace{\begin{cases}
    \Variance[\theta \sim LOO]{\derivedQtheta(z)} \\
    \Variance[\theta \sim Bootstrap]{\derivedQtheta(z)}
\end{cases}}_{\text{Frequentist Notions of Variance}}
\end{equation*}
Here the \emph{Delta} $\deltaUz \defeq \gradtheta \derivedQtheta(z)$ is the gradient vector of $\derivedQtheta$ evaluated at the input $z$. $\PosteriorCovarianceMatrix$ is a suitable matrix for which the canonical choice is an approximation of the scaled inverse Fisher Information matrix of $\ftheta$.

\paragraph{Canonical Choice of $\mathbf{\PosteriorCovarianceMatrix}$} 
The family of Delta Variances in principle supports any positive definitive matrix $\PosteriorCovarianceMatrix$. We will see in Section~\ref{sec:delta_bayesian_derivation} that it intuitively represents the posterior covariance of the parameters $\theta$ after learning $\ftheta$ on the training data $\Data$.
The canonical choice is $\PosteriorCovarianceMatrix\defeq\frac{1}{N}\hat{F}_{f}^{-1}$ being the inverse empirical Fisher Information matrix scaled by the number of data-points $N$.
Plugged into the Delta Variance formula we obtain the following estimate for the Epistemic Variance of $\derivedQtheta(z)$:
$$\Variance[\theta \sim p(\theta|f, \Data)]{\derivedQtheta(z)} \approx \frac{1}{N} \quadForm[\hat{F}_f^{-1}]{\gradtheta \derivedQtheta(z)}$$
It is worth emphasizing that the Fisher information is computed using $\fthetax$ (the model that was used for training $\theta$) while the gradient delta vectors come from $\derivedQtheta(z)$ the quantity of interest that is evaluated. Hence $\hat{F}_f^{-1}$ can be precomputed and reused for various choices of $\derivedQtheta$.

\paragraph{Intuition}
Section~\ref{sec:delta_analysis} explores multiple ways to theoretically justify the Delta Variance family. The Bayesian intuition is that $\PosteriorCovarianceMatrix$ captures the posterior covariance of the parameters $\theta$ while $\deltaUz=\gradtheta \derivedQtheta(z)$ translates this parameter uncertainty from variations in $\theta$ to variations in $\derivedQtheta(z)$. 
In Figure~\ref{fig:delta_variance_convergence} we consider an illustrative example, where a survival rate of $\fthetax = \theta$ has been estimated and is used to make predictions 10 years ahead via $\derivedQtheta(z)=\theta^{10}$.

\paragraph{Theoretical Motivation} The family of Delta Variance estimators is motivated because under strong conditions (see Section~\ref{sec:delta_bayesian_derivation}) and for number of data-points $N$ it can be shown to recover the Epistemic Variance up to a diminishing error:
$$ \underbrace{\Variance[\theta \sim p(\theta|f, \Data)]{\derivedQtheta(z)}}_{\text{Epistemic Variance}} = \underbrace{\quadForm[\PosteriorCovarianceMatrix]{\deltaUz}}_{\text{Delta Variance}} + \Bigoh{N^{-1.5}}$$
An additional motivation is that it can be derived using mild assumptions from a leave-one-out, statistical bootstrapping or an adversarial robustness perspective (see Sections \ref{sec:delta_frequentist_derivation} and \ref{sec:delta_adverserial_derivation}).

\paragraph{Computational Convenience}
Delta Variances are convenient because $\deltaUz \defeq \gradtheta \derivedQtheta(z)$ can be computed using any auto-differentiation framework and because $\PosteriorCovarianceMatrix$ does not depend on $\derivedQtheta$ (e.g. can be re-used for many different quantities of interest $\derivedQtheta$). It is efficient because it is a vector-matrix-vector product, where the matrix can be approximated efficiently (e.g. diagonally, low-rank, or using KFAC~\cite{Martens:2014}).

\begin{figure}[ht]
    \centering
    \includegraphics[width=.47\textwidth]{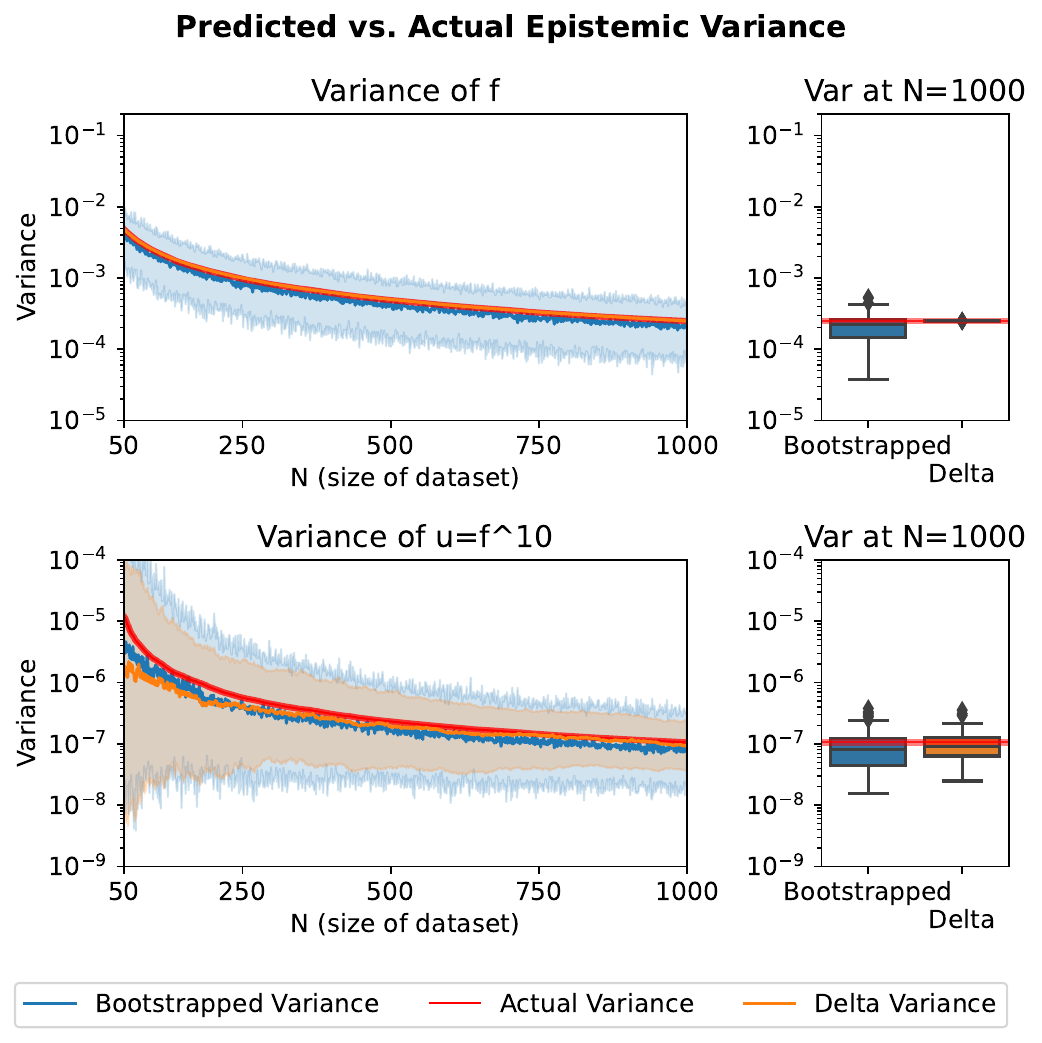}
    \caption{Illustrative survival prediction example. 
    Actual epistemic variance (red) vs. predicted variance using the Delta Variance (orange) or a 10-fold Bootstrap (blue) as the dataset size $N$ grows. Shaded confidence areas contain $95\%$ of the variance predictions. Bold lines are the median. Observe that the orange median line of the Delta Variance and the actual variance in red overlap largely.
    Top: variance of learned function $\ftheta(x)=\theta$ Bottom: variance of quantity of interest $\derivedQtheta(x) \defeq
    \theta^{10}$ evaluations. All methods yield reasonable results for $N>10$ with ensemble methods exhibiting higher variance. Generally the variance for $\derivedQtheta$ is harder to estimate than for $\ftheta$. 
    }
    \label{fig:delta_variance_convergence}
\end{figure}

\newcommand{\footBVM}{5}
\newcommand{\footBoot}{4}
\newcommand{\footLaplace}{3}
\newcommand{\footBoundedHessian}{2}
\newcommand{\footLocalConvergence}{1}
\begin{table*}
\centering
    \begin{tabular}{c|c|c}
        Bayesian Interpretations & Frequentist Interpretations & Choice of $\PosteriorCovarianceMatrix$ \\
        (Section~\ref{sec:delta_bayesian_derivation}) &  & (modulo factors of $N$) \\
        \toprule
        Bernstein-von Mises Posterior\footnotemark[\footBVM] & OOD Detection\footnotemark[\footLocalConvergence] (Sec.~\ref{sec:delta_OOD_derivation}) & $F^{-1}$ \\
        \rule{0pt}{16pt}  
        Misspecified Bernstein-von Mises Posterior\footnotemark[\footBVM] & 
        \makecell{
        Leave-one-out\footnotemark[\footBoundedHessian] \& Bootstrapped\footnotemark[\footBoot]}  Variance (Sec.~\ref{sec:delta_frequentist_derivation})
         & $H^{-1} F H^{-1}$ \\
        \rule{0pt}{14pt}  
        Laplace Posterior\footnotemark[\footLaplace] & Adversarial Robustness\footnotemark[\footBoundedHessian] (Sec.~\ref{sec:delta_adverserial_derivation}) & $H^{-1}$ 
    \end{tabular}
\caption{For each choice of $\PosteriorCovarianceMatrix$ there exists a Bayesian and a frequentist interpretation. Each interpretation requires different assumptions on $\fthetax$ and its loss. Due to their milder assumptions the frequentist interpretations can serve as fall-back interpretations if the stricter conditions on the Bayesian interpretations are not met. 
    For example observe that assuming a Bernstein-von Mises Posterior is computationally equivalent to performing OOD Detection. Interestingly the former makes strong assumptions about $\ftheta$ which typically do not apply to neural networks, while the later only requires differentiability of $\ftheta$ and $\derivedQtheta$ and that $\theta$ converges locally.
    The Hessian is computed with respect to the training loss of $\ftheta$. The Bayesian interpretations start with the true Fisher matrix and approximate it from data e.g. with $\hat{F}$. The frequentist approximations work with $\hat{F}$ directly. In practice both $H$ and $\hat{F}$ are computed at the locally optimal parameters $\thetabar$. We consider $\Sigma$ modulo factors of $N$ as they do not change the interpretation.
    }
\label{tab:delta_interpretations}
\end{table*}

\footnotetext[\footLocalConvergence]{
Only assumes differentiable $\derivedQtheta$ and locally converged $\thetabar$.}
\footnotetext[\footBoundedHessian]{Assumes: Bounded second derivatives of $\log\pfthetax$ and $\derivedQtheta$ with respect to $\theta$, Hessian at $\thetabar$ is invertible, and all from (\footLocalConvergence).}
\footnotetext[\footLaplace]{Assumes: All above, well-specified model $\pfthetax$, $\thetabar$ converged to a unique optimum of the posterior with locally quadratic shape.
Often referred to as \emph{Laplace Approximation} when some conditions are not met. See \citet{MacKay:92b} for more details.}
\footnotetext[\footBoot]{Assumes: All above except the local shape requirement on the posterior in (\footLaplace). Bootstrapping requires regularity of the estimator and see \citet{vandervaart1998:asymptotic} for more conditions.}
\footnotetext[\footBVM]{Similar to (\footBoot) but with stronger requirements: a uniformly consistent maximum likelihood estimator -- see \citet{vandervaart1998:asymptotic} for more conditions.}

\subsection{How to choose $\mathbf{\PosteriorCovarianceMatrix}$}
\label{sec:delta_variance_how_to_choose_sigma}

\paragraph{Principled choices of $\mathbf{\PosteriorCovarianceMatrix}$}
Theory suggests three principled choices for the covariance matrix, which all scale as $\Sigma\propto\frac{1}{N}$. Each choice can be derived in at least two ways using statistics or using influence functions (see Section~\ref{sec:delta_analysis} for details and Table~\ref{tab:delta_interpretations} for an overview). 
Note that here $H_f$ is the Hessian of the average loss.
\begin{enumerate}
    \item The inverse Fisher Information $F_f^{-1}$ divided by $N$.
    \item The inverse Hessian of the average training loss $\iHessian$ divided by $N$.
    \item The sandwich $\iHessian F_f \iHessian$ divided $N$.
\end{enumerate}
For well-specified models the three covariance matrices become eventually equivalent as the average Hessian $H_f$ and empirical Fisher converge to the true Fisher as data increases. In practice they need to be efficiently approximated from finite data (e.g. diagonally or using KFAC) and safely inverted. The first Bayesian approach uses $F_f$ which can be approximated by the empirical Fisher information $\hat{F}_f$ and is easily invertible with minuscule regularization as it is positive semi-definite by construction (eigenvalues $\geq 0$). For this reason the third Bayesian approach often uses $\hat{F}_f$ in place of $H$. Frequentist analogues use the empirical $\hat{F}_f$ directly. In contrast $H_f$ can only be inverted safely at a local minimum (eigenvalues $>0$), which may not be reached precisely with stochastic optimization. Hence inverting $H_f$ requires more careful regularization~\cite{Martens:2014}.
For simplicity we select $\Sigma$ to be a diagonal approximation of the empirical Fisher in our experiments, which alleviates the question of regularization.

\paragraph{A minimal pragmatic choice for $\mathbf{\PosteriorCovarianceMatrix}$}
Many popular optimizers such as ADAM track a diagonal approximation of the Fisher information to conduct a second-order-like optimization -- essentially normalizing the gradient magnitude.
Reusing the same diagonal approximation of $\Sigma$ for optimization and for the Delta Variance computation saves computational resources and lines of code. This is what we do in the minimal working example (see Algorithm~\ref{alg:delta_variance_minimal_example}).

\begin{algorithm}[h]
\caption{\textbf{Minimal Working Example} for Delta Variance estimation of a regression estimator $\fthetax=y$.\\
This minimum working example requires no changes to the training loop. It uses the ADAM second moment estimates as a diagonal Fisher information matrix approximation. Epistemic variance can be estimated for any differentiable $\derivedQtheta$; naturally $\derivedQtheta=\ftheta$ is possible. For alternative $\Sigma$ approximations beyond this minimal example consider Section~\ref{sec:delta_variance_how_to_choose_sigma}.\label{alg:delta_variance_minimal_example}}
{\color{gray}
\textbf{Regular Training Step:} (no modifications needed) \\
This is update step number $t$ with sampled batch $B_t \subset \Data$. \\
The ADAM optimizers internal state includes first and second moments: $m_t$ and $\nu_t$.
\begin{algorithmic}[1]
\State Define batched loss: $L_t(\theta) \defeq \sum_{(x, y) \in B_t} \frac{1}{2}(\fthetax-y)^2$
\State Compute gradient $\Delta_t \defeq \gradtheta L_t(\theta)$
\State ADAM update $\theta_t \defeq ADAM(\theta_{t-1}, \Delta_t) $
\end{algorithmic}
}
\textbf{Epistemic Variance Estimation:} \\
Inputs: Quantity of interest $\derivedQtheta(z)$ with input $z$.
Second moment vector $\nu$ from ADAM (e.g. at the end of training $t=T$), regularization $\epsilon$. Normalization factor $N$ representing the dataset size (optional, $N=1$ yields unnormalized estimates).
\begin{algorithmic}[1]
\State Let vector $\Delta \defeq \gradtheta \derivedQtheta(z)$.
\State Let scalar $v \defeq \sum_i \frac{\Delta_i^2}{\nu_i+\epsilon}$ (summation over all entries indexed by $i$)
\State Return scalar $v/N$ as an estimate of $\Variance[\theta]{\derivedQtheta(z)}$
\end{algorithmic}
\end{algorithm}

\paragraph{Fine-tuning or Learning $\mathbf{\PosteriorCovarianceMatrix}$}
The analytic form of the Delta Variance permits to back-propagate into the values of $\PosteriorCovarianceMatrix$.
This enables approaches that learn better values for $\PosteriorCovarianceMatrix$ from scratch or improve the values via fine-tuning. We explore a simple example in Section~\ref{sec:delta_learning_sigma} that improves empirically over the regular Fisher information by re-scaling some of its entries.

\section{Analysis}
\label{sec:delta_analysis}

In this section we will investigate multiple ways to derive and motivate Delta Variances. Broadly speaking they can be separated into three classes: 
\begin{enumerate}
    \item 
In Section~\ref{sec:delta_bayesian_derivation} we begin with the easiest derivations, which approximate the Bayesian posterior and make strong assumptions that may not always apply to neural networks. 
\item
In Section~\ref{sec:delta_frequentist_derivation} we consider the frequentist analogue of Epistemic Variance, that is compatible with neural networks and does not make assumptions about any posterior. 
\item In Section~\ref{sec:delta_adverserial_derivation} and \ref{sec:delta_OOD_derivation} we consider alternative derivations that are based on adversarial robustness and out-of-distribution detection and rely on even fewer assumptions.
\end{enumerate}
All of the considered derivations yield Delta Variances with principled covariance matrices. 
For an overview consider Table~\ref{tab:delta_interpretations}, where we can observe that assuming a Bernstein-von Mises Posterior is computationally equivalent to performing OOD Detection. Interestingly the former makes strong assumptions about $\ftheta$ which typically do not apply to neural networks, while the later only requires that the covariance of gradients is finite and that $\theta$ converges locally. Due to their milder assumptions the frequentist interpretations can serve as fall-back interpretations if the stricter conditions on the Bayesian interpretations are not met.

\subsection{Bayesian Interpretation}
\label{sec:delta_bayesian_derivation}
We begin with a derivation that gives rise to a bound on the approximation error. While requiring strong assumptions, it serves as a motivation and introduction. The error diminishes with the number of observed data-points $N$:
$$ \underbrace{\Variance[\theta \sim p(\theta|f, \Data)]{\derivedQtheta(z)}}_{\text{Epistemic Variance}} = \underbrace{\quadForm[\PosteriorCovarianceMatrix]{\deltaUz}}_{\text{Delta Variance}} + \Bigoh{N^{-1.5}}$$

\subsubsection{Bernstein-von Mises Motivation}
As a motivational introduction we will derive the approximation error when the Bernstein-von Mises conditions are met (e.g. differentiability and unique optimum -- see \citet{vandervaart1998:asymptotic} for details). Under such conditions the posterior converges to a Gaussian distribution centered around the maximum likelihood solution $\theta$ with a scaled inverse Fisher Information as covariance matrix.
$$ P(\theta|\Data)\to \normal{\theta, \frac{1}{N}F_f^{-1}}$$
The Epistemic Variance can then be computed using the Delta Method resulting in Proposition~\ref{prop:delta_bound}.
\begin{proposition}
\label{prop:delta_bound}
For a normally distributed posterior with mean $\thetabar$ and a covariance matrix $\PosteriorCovarianceMatrix$ proportional to $\frac{1}{N}$ 
it holds:
$$ \underbrace{\Variance[\theta \sim p(\theta|f, \Data)]{\derivedQtheta(z)}}_{\text{Epistemic Variance}} = \underbrace{\quadForm[\PosteriorCovarianceMatrix]{\deltaUz}}_{\text{Delta Variance}} + \Bigoh{N^{-1.5}}$$
where $\deltaUz\defeq \gradtheta \derivedQtheta(z)|_{\theta=\thetabar}$ as usual.
\end{proposition}
\begin{proof}
See appendix.
\end{proof}
If the Bernstein-von Mises conditions are met Proposition~\ref{prop:delta_bound} holds with $\PosteriorCovarianceMatrix=\frac{1}{N}F_f^{-1}$.

\paragraph{Further Bayesian Interpretations}
Other Gaussian posterior approximations can be considered by plugging their respective posterior covariance matrix into Proposition~\ref{prop:delta_bound}:
The misspecified Bernstein-von Mises theorem \cite[see][]{Kleijn:2012misspecifiedBvM} states that we obtain the sandwich covariance $\frac{1}{N} \iHessian F_f \iHessian$ if the model $\pfthetax$ is misspecified (i.e. does not represent the data well). Proponents advocate that the sandwich estimate is more robust to heteroscedastic noise while others argue against it~\cite{Freedman:2006SandwichEstimator}. Similarly a Laplace approximation \cite{laplace1774:integral,MacKay:92b,ritter2018:laplace} can be made resulting in a Delta Variance with $\frac{1}{N}\iHessian$.
Again those choices of $\Sigma$ are $\propto\frac{1}{N}$.

\subsection{Frequentist Interpretation}
\label{sec:delta_frequentist_derivation}
To better cater to complex function approximators such as neural networks this section discusses two frequentist derivations of the Delta Variance, which rely on milder assumptions: As they are frequentist they do not consider posterior distributions. This allows us to side-step any questions about the shape and tractability of posterior distributions for neural networks. 
In the first case (leave-one-out interpretation) we even do not require global convexity or a unique optimum. Convergence of the parameters to some local optimum together with locally bounded second derivatives is sufficient.

\subsubsection{Leave-One-Out Variance}
\label{sec:delta_analysis_loo_variance}
In Proposition~\ref{prop:delta_infinitessimal} we observe that the Delta Variance computes an infinitesimal approximation to the leave-one-out variance (see Definition~\ref{def:delta_infinitessimal_variance}) for choice of $\PosteriorCovarianceMatrix=\iHessian \hat{F}_f \iHessian$:

$$\underbrace{\quadForm[\PosteriorCovarianceMatrix]{\deltaUz}}_{\text{Delta Variance}} = \underbrace{\Variance[\theta \sim IJ]{\derivedQtheta(z)}}_{\text{Infinitesimal LOO Variance}} \approx \underbrace{\Variance[\theta \sim LOO]{\derivedQtheta(z)}}_{\text{LOO Variance}} $$
The infinitesimal approximation to the leave-one-out variance (also known as the infinitesimal jackknife \cite{Jaeckel:1972InfinitesimalJackknife}) is defined as follows:
\begin{definition}
Let $\theta_i$ be the parameters resulting from training $\ftheta$ on data $\Data$ with a single data-point $x_i$ down-weighted from weight $\frac{1}{N}$ to $\frac{1-\epsilon}{N}$, then the \emph{$\epsilon$-Leave-One-Out Variance} is defined as
$$ \Variance[\theta \sim \epsilon-LOO]{\derivedQtheta(z)} \defeq \frac{N}{\epsilon^2} \Variance[\iinN]{\derivedQ_{\theta_i}(z)} $$
\end{definition}

\begin{definition}
\label{def:delta_infinitessimal_variance}
With slight abuse of notation we define the \emph{Infinitesimal LOO Variance} as the limit of the $\epsilon$-Leave-One-Out Variance:
$$ \Variance[IJ]{\derivedQtheta(z)} \defeq \lim_{\epsilon \to 0} \Variance[\theta \sim \epsilon-LOO]{\derivedQtheta(z)}$$
\end{definition}

\begin{proposition}
\label{prop:delta_infinitessimal}
The Delta Variance equals the infinitesimal LOO Variance for $\PosteriorCovarianceMatrix=\frac{1}{N}\iHessian \hat{F}_f \iHessian$:
$$ \Variance[IJ]{\derivedQtheta(z)} = \quadForm[\PosteriorCovarianceMatrix]{\deltaUz} $$
In particular
$$ \Variance[\theta \sim \epsilon-LOO]{\derivedQtheta(z)} = \quadForm[\PosteriorCovarianceMatrix]{\deltaUz} + \bigoh\left(\frac{\epsilon}{N^2}\right)  $$
\end{proposition}
\begin{proof}
See appendix.
\end{proof}

\subsubsection{Bootstrapped Variance}
\label{sec:delta_analysis_boot_variance}
Here we explain how statistical bootstrapping can be approximated using Delta Variances. 
In Proposition~\ref{prop:delta_boostrapping} we observe that the Delta Variance computes an infinitesimal approximation to statistical bootstrapping for choice of $\PosteriorCovarianceMatrix=\frac{1}{N} \iHessian \hat{F}_f \iHessian$:
$$\underbrace{\quadForm[\PosteriorCovarianceMatrix]{\deltaUz}}_{\text{Delta Variance}} = \underbrace{\Variance[\theta \sim IB]{\derivedQtheta(z)}}_{\text{Infinitesimal Bootstrap Variance}} \approx \underbrace{\Variance[\theta \sim Bootstrap]{\derivedQtheta(z)}}_{\text{Bootstrap Variance}} $$
Bootstrapping trains a sequence of new models on random subsets of the dataset and evaluates their distribution in lieu of the Bayesian posterior. For example it may  train new models on random halves of the dataset. This procedure corresponds to assigning each data-point a random weight $w$ of either $0$ or $2$ with equal probability. Here we consider a generalization  
where $w$ is $1-\epsilon$ or $1+\epsilon$ which resembles the halving procedure for $\epsilon$ equals one. Interestingly this bootstrapping scheme can be approximated using Delta Variances.

\begin{definition}
Let random variable $\thetaRademacher$ be the parameters obtained from training $\ftheta$ on data $\Data$ with each data-point $x\in \Data$ re-weighted iid. randomly to either $\frac{1-\epsilon}{N}$ or $\frac{1+\epsilon}{N}$. Let those weights be described using vector $r$ (where all $r_i$ are Rademacher distributed): $\frac{1+r_i\epsilon}{N}$,
    then
then the \emph{$\epsilon$-Bootstrapped Variance} is defined as
$$ \Variance[\theta \sim \epsilon-Bootstrap]{\derivedQtheta(z)} \defeq \frac{1}{\epsilon^2} \Variance[\bootRademacher]{\derivedQ_{\thetaRademacher}(z)} $$
\end{definition}

\begin{proposition}
\label{prop:delta_boostrapping}
The Delta Variance equals the $\epsilon$-Bootstrapped Variance with a diminishing approximation error for $\PosteriorCovarianceMatrix=\frac{1}{N} \iHessian \hat{F}_f \iHessian$:
$$ \Variance[\theta \sim \epsilon-Bootstrap]{\derivedQtheta(z)} = \quadForm[\PosteriorCovarianceMatrix]{\deltaUz} + \bigoh\left( \frac{\epsilon}{N^{2.5}} \right)$$
\end{proposition}
\begin{proof}
    Appendix
\end{proof}

\subsection{Adversarial Data Interpretation}
\label{sec:delta_adverserial_derivation}
Sometimes it is of interest to quantify how much a prediction changes if the training dataset is subject to adversarial data injection. 
Intuitively this is connected to epistemic uncertainty: one may argue that predictions are more robust the more certain we are about their parameters and vice versa.
In Appendix~\ref{sec:delta_adverserial_derivation_appendix} we show that this intuition also holds mathematically.
In particular we observe that:
\begin{enumerate}
    \item The Delta Variance with $\PosteriorCovarianceMatrix=\frac{1}{N}\iHessian$ computes how much a quantity of interest $\derivedQtheta(z)$ changes if an adversarial data-point is injected.
    \item This adversarial interpretation is technically equivalent to the Laplace Posterior approximation (from Section~\ref{sec:delta_bayesian_derivation}) -- even though interestingly both start with different assumptions and objectives. 
\end{enumerate}

\subsection{Out-of-Distribution Interpretation}
\label{sec:delta_OOD_derivation}
We show that a large Delta Variance of $\derivedQtheta(z)$ implies that its input $z$ is out-of-distribution with respect to the training data. This relates to epistemic uncertainty intuitively: a model is likely to be uncertain about data-points that differ from its training data. The derivation in Section~\ref{sec:delta_OOD_derivation_appendix} is based on the Mahalanobis Distance \cite{Mahalanobis:1936Distance} -- a classic metric for out of distribution detection.
It accounts for the choice of $\ftheta\neq\derivedQtheta$ and relies on minimal assumptions only requiring existence of gradients and that the training of $\ftheta$ has converged.

\section{Experiments}
\label{sec:delta_experiments}
To empirically study the Delta Variance we build on the state-of-the-art GraphCast weather forecasting system \cite{lam:2023GraphCast} which trains a neural network $\fthetax$ to predict the weather 6 hours ahead. This $\fthetax$ is then iterated multiple times to make predictions up to 10 days into the future. We define various quantities of interest $\derivedQtheta$ such as the average rainfall in an area or the expected power of a wind turbine at a particular location and compute their Epistemic Variance. We assess the Epistemic Variance predictions on 4 years of hold-out data using multiple metrics such as the correlation between predicted variance and prediction error and the likelihood of the quantities of interest. Empirically Delta Variances with a diagonal Fisher approximation yield competitive results at lower computational cost -- see Figure~\ref{fig:delta_pareto_quality_vs_inference}.
Next we give an overview on the experimental methodology -- please consider the appendix for more technical details.

\begin{figure}[ht!]
    \centering
    \includegraphics[width=.45\textwidth]{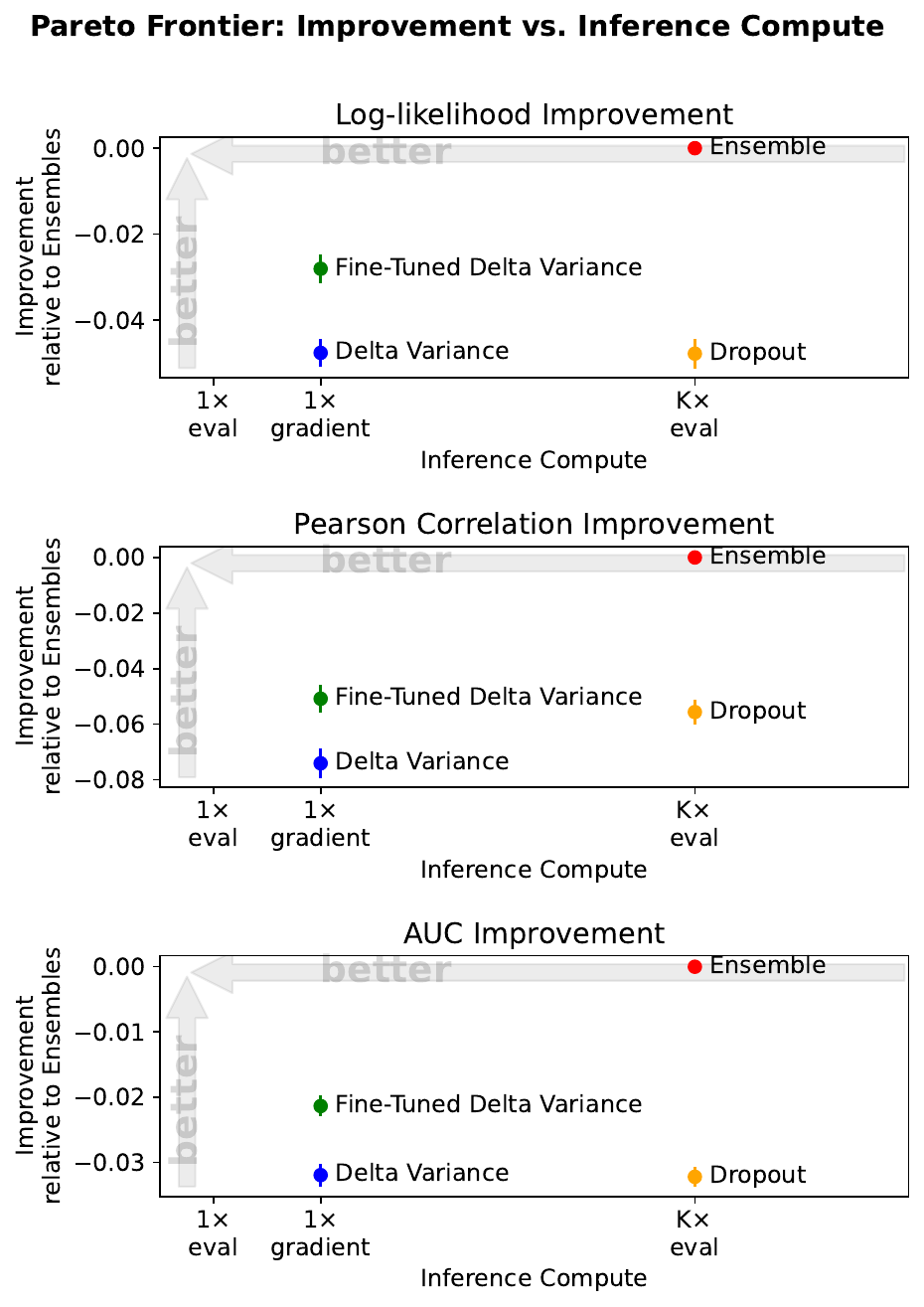}
    \caption{
    Comparison of variance estimators in terms of their inference cost and prediction quality. The quantities of interest are based on the GraphCast \cite{lam:2023GraphCast} weather prediction system that iterates a learned neural network dynamics model to form predictions.
    We evaluate the selected variance estimators based on three different evaluation criteria (Log-likelihood, correlation to prediction error and AUC akin to \citet{Amersfoort:2020UncertaintyDUQ}). Lines indicate 2 standard errors.
    Delta Variances yield similar results as popular alternatives for lower computational cost.
    On average ensembles achieve the highest quality and Delta Variances the lowest computational overhead. See Section~\ref{sec:delta_learning_sigma} for the fine-tuned Delta Variance.
    }
    \label{fig:delta_pareto_quality_vs_inference}
\end{figure}

\subsection{Weather Forecasting Benchmark}

\paragraph{GraphCast Training}
We build on the state-of-the-art GraphCast weather prediction system. It trains a graph neural network to predict the global weather state 6 hours into the future. This step function $x_{t+1}=\ftheta(x_t)$ is then iterated to predict up to 10 days into the future. The global weather state $x$ is represented as a grid with 5 surface variables and 6 atmospheric variables at 37 levels of altitude \citep[see][for details]{lam:2023GraphCast}. The authors consider a grid-sizes of $0.25$ degrees. To save resources we retrain the model for a grid size of $4$ degrees and reduce the number of layers and latents each by factor a of $2$.
Finally we skip the fine-tuning curriculum for simplicity. 
Besides the graph neural network we also consider a standard convolutional neural network. Training data ranges from 1979-2013 with validation data from 2014-2017 and holdout data from 2018-2021 resulting in about 100 GB of weather data.

\paragraph{Quantities of Interest}
First we define $126$ different quantities of interest $\derivedQtheta$ based on $4$ topics that we evaluate on the hold-out data (2018-2021) for two different neural network architectures: 1) Precipitation at various times into the future. 2) Inspired by wind turbine energy yield we measure the third power of wind-speed at various times into the future. 3) Inspired by flood risk we measure precipitation averaged over areas of increasing size five days into the future. 4) Inspired by health emergencies we predict the maximum temperature maximized over areas of increasing size five days into the future.
The first two quantities are predicted $1, \ldots, 5$ days ahead. The last two are measured 5 days ahead in quadratic areas with inradii ranging from $1$ to $6$. These measurements take place at $6$ preselected capital cities. 
Finally note that we never train $\derivedQtheta$ as it can be derived using $\ftheta$.

\paragraph{Evaluation Methodology}
The data from 2018-2021 is held out for evaluation resulting in approximately $6 \times 10^3$ different (input, target-value) pairs for each of the 252 quantities of interest $\derivedQtheta$. For each pair we obtain a prediction error $|y-\derivedQtheta(z)|$ and corresponding variance predictions $\nu(z)$.
Unfortunately many practical applications do not admit ground truth values for Epistemic Variance that one could compare variance estimators to. Instead there are multiple popular approaches in the literature relying on the prediction error, which is subject to both epistemic and aleatoric uncertainty.
In Figure~\ref{fig:delta_pareto_quality_vs_inference} we consider multiple such different criteria: 

\begin{enumerate}
    \item Akin to \citet{Amersfoort:2020UncertaintyDUQ} AUC considers how fast the average $L_1$ error decreases when data-points are removed from the dataset -- in the order of their largest predicted Epistemic Variance. 
    \item  We consider the Pearson correlation between absolute error and predicted epistemic standard deviation: $corr(|y-\derivedQtheta(z)|, \sqrt{\nu(z)})$.
    \item To evaluate the Log-likelihood of observations $y$ we interpret $\derivedQtheta(z)$ as the mean of a Laplace distribution with variance derived from the predicted Epistemic Variance $\nu(z)$. We parameterize the Laplace distribution such that its variance decomposes in a constant $\alpha$ and the predicted Epistemic Variance $\nu(z)$ scaled by $\beta$. Intuitively $\alpha$ represents the aleatoric variance and $\beta\nu(z)$ represents the Epistemic Variance: $\mathrm{Laplace}(\mu=\derivedQtheta(z), 2b^2=\alpha + \beta\nu(z))$. Both $\alpha$ and $\beta$ are learned on the validation data (2014-2017) that is used for hyper-parameter selection. We then observe how well it models the actual observed target values $y$ from the evaluation data (2018-2021).
\end{enumerate}

Finally to reduce variance we define the \emph{Improvement} of a variance estimator as the difference of its score to the score obtained by the ensemble estimator. Intuitively this indicates the loss in Quality when using an estimator in place of an ensemble. This procedure is repeated for each of the 252 quantities of interest $\derivedQtheta$.

\section{Illustrations and Extensions}
\label{sec:delta_illustations_and_extensions}
To highlight the generality of our approach we illustrate two extensions in this section.
\begin{enumerate}
    \item By learning $\PosteriorCovarianceMatrix$ to represent uncertainty well, we generalize the parametric from of Delta Variances beyond Fisher and Hessian matrices and observe improved results in the GraphCast benchmark -- see Figure~\ref{fig:delta_pareto_quality_vs_inference}.
    \item We consider an example where $\derivedQtheta$ is not an explicit function but maps to a fixed-point of an iterative algorithm. We observe that it is possible to compute the Delta Variance of fixed-points using the implicit function theorem.
    Applied to an eigenvalue solver we observe empirically that the Delta Variance yields reasonable uncertainty estimates -- see Figure~\ref{fig:delta_finite_element_eigenvalue}.
\end{enumerate}

\subsection{Learning $\mathbf{\PosteriorCovarianceMatrix}$}
\label{sec:delta_learning_sigma}
In Section~\ref{sec:delta_analysis} we observed that Delta Variances with special $\PosteriorCovarianceMatrix$ such as the Fisher Information approximate theoretically established measures of uncertainty. In this section we observe that $\PosteriorCovarianceMatrix$ may also be learned or fine-tuned. In an illustrative example we differentiate the Delta Variances with respect to $\PosteriorCovarianceMatrix$ and use gradient descent to obtain an improved $\PosteriorCovarianceMatrix$.
This may be helpful to improve the uncertainty prediction or to improve a downstream use-case if the variance is used in a larger system.

\paragraph{Fine-Tuning $\mathbf{\PosteriorCovarianceMatrix}$ Example}
We present a simple instance of fine-tuning a few parameters of $\PosteriorCovarianceMatrix$, which empirically yields improved results -- see Figure~\ref{fig:delta_pareto_quality_vs_inference}. Note that $\Sigma$ is approximated block-diagonally in most practical cases to limit the computational requirements -- with one block for each weight vector in each neural network layer. Hence the Delta Variance splits into a sum of per-block Delta Variances derived from per-block gradients $\Delta_i$:
$$ \quadForm[\PosteriorCovarianceMatrix]{\deltaF} = \sum_i  \quadForm[\PosteriorCovarianceMatrix_i]{\Delta_i}$$
In this example we introduce a factor to rescale $\Sigma$ within each block. Intuitively this adjusts the importance of each layer. 
Since only a few parameters need to be estimated we only need little fine-tuning data. This is applicable in situations where there is a small amount of training data for $\derivedQtheta$. In our experiments we optimize the coefficients of this linear combination using gradient descent to improve the log-likelihood or correlation on a small set of held-out validation data. Note that the per-layer variances can be cached which reduces the optimization problem significantly.

\subsection{Epistemic Variance of Iterative Algorithms and Implicit Functions}

So far we considered quantities of interest $\derivedQtheta$ that are explicit functions of the parameters $\theta$. Here we consider an example where the quantity of interest is an implicit function: $\derivedQtheta$ maps to the fixed-point (solution) of an iterative algorithm for which there is no closed-form formula that we could differentiate to obtain its gradient.

Given some initial point $\ziterate_0$ the iteration  $\ziterate_{k+1}=\fixedpointSelfFunction_\theta(\ziterate_{k})$ may converge to a fixed-point $\zfixed_\theta$ that depends on the parameters $\theta$. To estimate $\Variance[\theta]{\zfixed_\theta}$ we need to define $\derivedQtheta$ as follows, which can not be differentiated with regular back-propagation due to the limit
$$ \derivedQtheta(\ziterate_0) \defeq \lim_{k\to\infty} \underbrace{\fixedpointSelfFunction_\theta \circ \dots \circ \fixedpointSelfFunction_\theta}_{k \text{ times}}(\ziterate_0) = \zfixed_\theta
 $$

\paragraph{Implicit Epistemic Variance Calculation}
To compute the Delta Variance of an implicitly defined $\derivedQtheta$ we need its gradient $\gradtheta\derivedQtheta$. This can be obtained under mild conditions using the implicit function theorem.
Let us denote $\ziterate_{k+1}=\fixedpointSelfFunction_\theta(\ziterate_{k})$ any fixed-point iteration converging to $\zfixed$ with the corresponding non-linear equation $\fixpointZeroFunction_\theta(\ziterate)\defeq\fixedpointSelfFunction_\theta(\ziterate)-\ziterate=0$. The implicit function theorem yields the gradient of $\derivedQtheta$ by considering the Jacobian of $\fixpointZeroFunction$ at the fixed-point $\zfixed$:
\begin{align*}
    \Delta_{\zfixed} \defeq \gradtheta &\derivedQtheta = -\left(\nabla_{\zfixed} \fixpointZeroFunction_\theta(\zfixed)  \right)^{-1}\gradtheta \fixpointZeroFunction_\theta(\zfixed)
\end{align*}
whenever $\fixpointZeroFunction$ is continuously differentiable and the inverse of $\nabla_{\zfixed} \fixpointZeroFunction_\theta(\zfixed)$ exists. Now we can compute the Epistemic Variance as $\Variance[\theta]{\zfixed}\approx\quadForm[\PosteriorCovarianceMatrix]{\Delta_{\zfixed}}$.

\paragraph{Eigenvalue Example}
Eigenvalues are a quantity of interest in structural engineering. As an illustrative example we consider the eigenvalues $\lambda_i(\EigenvalueExampleMatrix_\theta)$ of a finite element model matrix $\EigenvalueExampleMatrix_\theta=M_\theta^{-1}K_\theta$ that indicates the stability of a physical structure. If the parameters $\theta$ are uncertain the eigenvalues will be uncertain as well. Recall that they are the solutions to $\det(\EigenvalueExampleMatrix_\theta-\lambda I)=0$. We can estimate the Epistemic Variance of an eigenvalue $\Variance[\theta]{\lambda_i(\EigenvalueExampleMatrix_\theta)}$ using Delta Variances if we obtain the gradient of the eigenvalue $\gradtheta \lambda_i(\EigenvalueExampleMatrix_\theta)$. To this extend we need the implicit function theorem as $\lambda_i(\EigenvalueExampleMatrix_\theta)$ is an implicitly defined function -- please consider the appendix for technical details. 

\begin{figure}[ht]
    \centering
    \includegraphics[width=.45\textwidth]{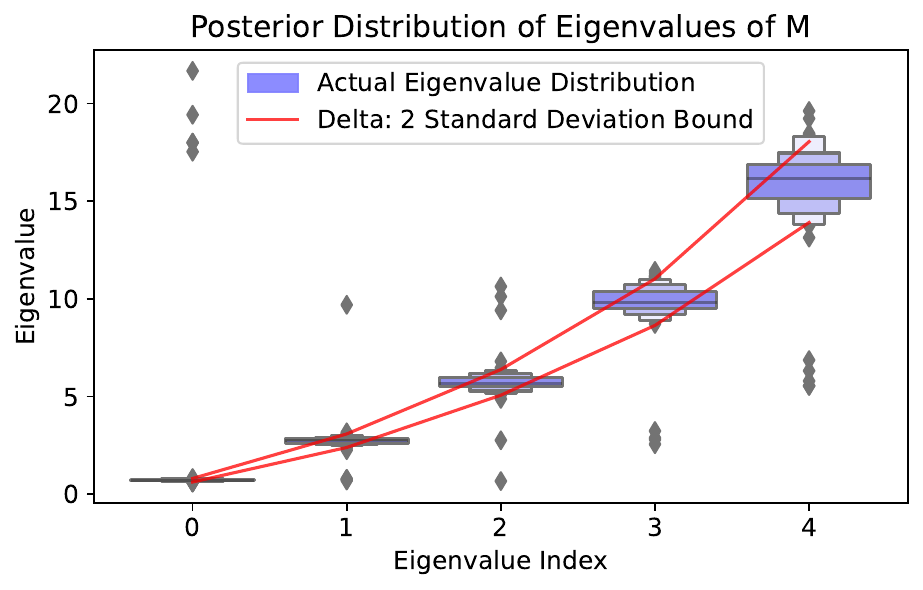}
    \caption{To investigate more intricate quantities of interest, we consider the mapping from a matrix $\EigenvalueExampleMatrix_\theta$ to its eigenvalue $\derivedQtheta=\lambda_i(\EigenvalueExampleMatrix_\theta)$. This function is not explicit and computed using iterative algorithms, but we can use the implicit function approach to estimate the Delta Variance. Here $\EigenvalueExampleMatrix_\theta$ is an illustrative finite-element problem with 11-dimensional parameters $\theta$ and 5 eigenvalues.
    }
    \label{fig:delta_finite_element_eigenvalue}
\end{figure}

\section{Related Work}
The proposed Delta Variance family bridges and extends Bayesian, frequentist and heuristic notions of variance. Furthermore it generalizes related work by considering explicit and implicit quantities of interest other than the neural network itself $\derivedQtheta \neq \ftheta$ and permits learning improved covariances $\PosteriorCovarianceMatrix$ -- see Section~\ref{sec:delta_illustations_and_extensions}. Below we give a brief historic account of related methods that mostly consider the $\derivedQtheta=\ftheta$ case.

\paragraph{Delta Method}
The Delta Method dates back to \citet{cotes:1722Harmonia}, \citet{Lambert:1765Beytraege} and \citet{gauss:1823ErrorTheory} in the context of error propagation and received a modern treatment by \citet{Kelley:1928CrossroadsMind,Wright:1934PathCoefficients, Doob:1935LimitingDistributions,Dorfman:1938note} -- see~\citet{Gorroochurn:2020WhoInventedDelta} for a historical account. \citet{Denker:1990Transforming} apply the Delta Method to the outputs of neural networks $\ftheta(x)$ and \citet{Nilsen:2022DeltaMethod} improves computational efficiency. When applied to neural networks the Delta Method requires strong assumptions about the posterior (e.g. unique optimum) or training process, which have not been proven to hold.
Delta Variances --  named after the Delta Method -- provide multiple alternative theoretical justification through its unifying perspective. Furthermore Delta Variances generalize to the $\derivedQtheta\neq\ftheta$ case and other $\PosteriorCovarianceMatrix$.

\paragraph{Laplace Approximation}
Building on work by \citet{Gull:1989MaximumEntropy}, \citet{MacKay:92b} and \citet{ritter2018:laplace} apply the Laplace approximation to neural networks. 
Approximating functions at an optimum by what should later be called a Gaussian distribution dates back to \citet{laplace1774:integral}. 
While only applicable to a single optimum \citet{MacKay:92b} heuristically argues for its applicability to posterior distributions of neural networks. Given such Gaussian posterior approximation they apply the Delta Method yielding a special instance of the Delta Variance family with $\derivedQtheta=\ftheta$ and $\Sigma=\iHessian$ -- see Section~\ref{sec:delta_bayesian_derivation}.

\paragraph{Influence Functions and Jackknife Methods}
Influence functions were proposed in \citet{Hampel:1974Influence} concurrently with the closely related Infinitesimal Jackknife by \citet{Jaeckel:1972InfinitesimalJackknife} which approximates cross validation \cite{Quenouille:1949Correlation}. \citet{Koh:2017UnderstandingBlackBox} apply the influence function analysis to neural networks to evaluate how training data influences predictions. In Sections~\ref{sec:delta_frequentist_derivation} and \ref{sec:delta_adverserial_derivation} we apply similar techniques to general quantities of interest different from $\ftheta$.

\paragraph{Statistical Bootstrapping}
Bootstrapping was proposed by \citet{Efron:1979} as an improvement to the Jackknife method, which can be seen as a linearization of the Bootstrapping procedure. We build on this connection and translate it to quantities of interest. \citet{Bickel1981:Bootstrap} provided theoretical justification for a wide class of bootstrapped statistics and estimators. Bootstrapping has been applied heuristically to neural networks by \citet{Lakshminarayanan2017:ensembles}.

\paragraph{Uncertainty Estimation for Deep Neural Networks}
We focus our comparison on two popular methods: \citet{Lakshminarayanan2017:ensembles} train multiple neural networks to form an ensemble and \citet{Gal2016:dropout} which re-interprets Monte-Carlo dropout as variational inference. In Table~\ref{tab:delta_benefits} we compare their properties with Delta Variances observing that they come at larger inference cost.
\citet{Osband2023:EpistemicNeuralNetworks} aims to reduce the training costs of ensemble methods. To this extent they change the neural network architecture and training procedure, however how to reduce the remaining $k$-fold inference cost and memory requirements remain open research questions. Other popular methods come with similar requirements to change the architecture or training procedure \citep{blundell2015:BBB,Amersfoort:2020UncertaintyDUQ,Immer:2021LocalLinearization},
while approaches like \citet{Sun:2022OOD} are of non-parametric flavour exhibiting inference cost that increases with the dataset size. 
SWAG \cite{Maddox2019:SWAG} reduces the training and memory cost by considering an ensemble of parameters from a single learning trajectory with stochastic gradient descent and approximating it with a Gaussian posterior. For inference they employ expensive k-fold sampling. We note that it is natural to derive a SWAG-inspired Delta Variances that employs the $\Sigma$ from SWAG inside the computationally efficient Delta Variance formula -- we leave those considerations for future research.
Finally
\citet{Kallus2022:Implicit} propose a Delta Method inspired approach to approximate Epistemic Variance with an ensemble of two predictors and \citet{Schnaus:2023LearningPriors} learn scale parameters in Gaussian prior distributions for transfer learning.

\section{Conclusion}
We have addressed the question of how the uncertainty from limited training data affects the computation of downstream predictions in a system that relies on learned components. 
To this extent we proposed the Delta Variance family, which unifies and extends multiple related approaches theoretically and practically. 
We discussed how the Delta Variance family can be derived from six different perspectives (including Bayesian, frequentist, adversarial robustness and out-of-distribution detection perspectives) highlighting its wide theoretical support and providing a unifying view on those perspectives.
Next we presented extensions and applications of the Delta Variance family such as its compatibility with implicit functions and ability to be improved through fine-tuning. Finally an empirical validation on a state-of-the-art weather forecasting system shows that Delta Variances yield competitive results more efficiently than other popular approaches.

\section*{Acknowledgements} 
We would like to thank Remi Lam for his advice and support with GraphCast-related questions. Furthermore we would like to thank Mark Rowland, Wojciech M. Czarnecki, Danilo J Rezende, Iurii Kemaev and the anonymous AAAI 2025 reviewers for their valuable feedback.

\bibliography{aaai25}

\begin{thebibliography}{53}
\providecommand{\natexlab}[1]{#1}

\bibitem[{Auer, Cesa-Bianchi, and Fischer(2002)}]{Auer:2002}
Auer, P.; Cesa-Bianchi, N.; and Fischer, P. 2002.
\newblock Finite-time analysis of the multiarmed bandit problem.
\newblock \emph{Machine learning}, 47(2-3): 235--256.

\bibitem[{Bernstein(1917)}]{Bernstein:TheoryProbabilities}
Bernstein, S.~N. 1917.
\newblock \emph{The Theory of Probabilities}.

\bibitem[{Bickel and Freedman(1981)}]{Bickel1981:Bootstrap}
Bickel, P.~J.; and Freedman, D.~A. 1981.
\newblock Some Asymptotic Theory for the Bootstrap.
\newblock \emph{The Annals of Statistics}, 9(6): 1196--1217.

\bibitem[{Blundell et~al.(2015)Blundell, Cornebise, Kavukcuoglu, and Wierstra}]{blundell2015:BBB}
Blundell, C.; Cornebise, J.; Kavukcuoglu, K.; and Wierstra, D. 2015.
\newblock Weight Uncertainty in Neural Network.
\newblock In Bach, F.; and Blei, D., eds., \emph{Proceedings of the 32nd International Conference on Machine Learning}, volume~37 of \emph{Proceedings of Machine Learning Research}, 1613--1622. Lille, France: PMLR.

\bibitem[{Cook and Weisberg(1982)}]{cook:1982residuals}
Cook, R.~D.; and Weisberg, S. 1982.
\newblock \emph{Residuals and Influence in Regression}.
\newblock Monographs on Statistics \& Applied Probability. Chapman \& Hall.
\newblock ISBN 9780412242809.

\bibitem[{Cotes(1722)}]{cotes:1722Harmonia}
Cotes, R. 1722.
\newblock \emph{Harmonia Mensurarum}.
\newblock Robert Smith.

\bibitem[{de~Veaux et~al.(1998)de~Veaux, Schumi, Schweinsberg, and Ungar}]{Veaux:1998PredictionIntervals}
de~Veaux, R.~D.; Schumi, J.; Schweinsberg, J.; and Ungar, L.~H. 1998.
\newblock Prediction Intervals for Neural Networks via Nonlinear Regression.
\newblock \emph{Technometrics}, 40(4): 273--282.

\bibitem[{Denker and LeCun(1990)}]{Denker:1990Transforming}
Denker, J.; and LeCun, Y. 1990.
\newblock Transforming Neural-Net Output Levels to Probability Distributions.
\newblock In Lippmann, R.; Moody, J.; and Touretzky, D., eds., \emph{Advances in Neural Information Processing Systems}, volume~3. Morgan-Kaufmann.

\bibitem[{Doob(1935)}]{Doob:1935LimitingDistributions}
Doob, J.~L. 1935.
\newblock The limiting distributions of certain statistics.
\newblock \emph{Ann. Math. Stat.}, 6(3): 160--169.

\bibitem[{Dorfman(1938)}]{Dorfman:1938note}
Dorfman, R. 1938.
\newblock A note on the delta-method for finding variance formulae.
\newblock \emph{Biometric Bulletin}.

\bibitem[{Duff(2002)}]{Duff:2002}
Duff, M. 2002.
\newblock \emph{Optimal Learning: Computational procedures for {Bayes}-adaptive {Markov} decision processes}.
\newblock Ph.D. thesis, University of Massachusetts Amherst.

\bibitem[{Efron(1979)}]{Efron:1979}
Efron, B. 1979.
\newblock Bootstrap methods: another look at the jackknife.
\newblock \emph{The annals of Statistics}, 7(1): 1--26.

\bibitem[{Freedman(2006)}]{Freedman:2006SandwichEstimator}
Freedman, D.~A. 2006.
\newblock On The So-Called “Huber Sandwich Estimator” and “Robust Standard Errors”.
\newblock \emph{The American Statistician}, 60(4): 299--302.

\bibitem[{Gal and Ghahramani(2016)}]{Gal2016:dropout}
Gal, Y.; and Ghahramani, Z. 2016.
\newblock Dropout as a Bayesian Approximation: Representing Model Uncertainty in Deep Learning.
\newblock In Balcan, M.~F.; and Weinberger, K.~Q., eds., \emph{Proceedings of The 33rd International Conference on Machine Learning}, volume~48 of \emph{Proceedings of Machine Learning Research}, 1050--1059. New York, New York, USA: PMLR.

\bibitem[{Gauss(1823)}]{gauss:1823ErrorTheory}
Gauss, C. 1823.
\newblock \emph{Theoria combinationis observationum erroribus minimis obnoxiae}.
\newblock H. Dieterich.

\bibitem[{Gorroochurn(2020)}]{Gorroochurn:2020WhoInventedDelta}
Gorroochurn, P. 2020.
\newblock Who invented the delta method, really?
\newblock \emph{Math. Intelligencer}, 42(3): 46--49.

\bibitem[{Gull(1989)}]{Gull:1989MaximumEntropy}
Gull, S.~F. 1989.
\newblock \emph{Developments in Maximum Entropy Data Analysis}, 53--71.
\newblock Dordrecht: Springer Netherlands.
\newblock ISBN 978-94-015-7860-8.

\bibitem[{Hampel(1974)}]{Hampel:1974Influence}
Hampel, F.~R. 1974.
\newblock The Influence Curve and Its Role in Robust Estimation.
\newblock \emph{Journal of the American Statistical Association}, 69(346): 383--393.

\bibitem[{Heger(1994)}]{Heger:94}
Heger, M. 1994.
\newblock Consideration of Risk in Reinforcement Learning.
\newblock In \emph{Machine Learning: Proceedings of the 11th International Conference}, 105--111. Morgan Kaufmann Publishers, San Francisco, CA.

\bibitem[{Hodges(1967)}]{Hodges1967:Efficiency}
Hodges, J.~L. 1967.
\newblock Efficiency in normal samples and tolerance of extreme values for some estimates of location.
\newblock In \emph{Proceedings of the Fifth Berkeley Symposium on Mathematical Statistics and Probability}, volume~1, 163--–186. Berkeley. University of California Press.

\bibitem[{Hwang and Ding(1997)}]{Hwang:1997PredictionIntervals}
Hwang, J. T.~G.; and Ding, A.~A. 1997.
\newblock Prediction Intervals for Artificial Neural Networks.
\newblock \emph{Journal of the American Statistical Association}, 92(438): 748--757.

\bibitem[{Immer, Korzepa, and Bauer(2021)}]{Immer:2021LocalLinearization}
Immer, A.; Korzepa, M.; and Bauer, M. 2021.
\newblock Improving predictions of Bayesian neural nets via local linearization.
\newblock In Banerjee, A.; and Fukumizu, K., eds., \emph{Proceedings of The 24th International Conference on Artificial Intelligence and Statistics}, volume 130 of \emph{Proceedings of Machine Learning Research}, 703--711. PMLR.

\bibitem[{Jaeckel(1972)}]{Jaeckel:1972InfinitesimalJackknife}
Jaeckel, L. 1972.
\newblock The Infinitesimal Jackknife.
\newblock \emph{Bell Lab. Memorandum}, MM72-1215-11.

\bibitem[{Kallus and McInerney(2022)}]{Kallus2022:Implicit}
Kallus, N.; and McInerney, J. 2022.
\newblock The Implicit Delta Method.
\newblock In Koyejo, S.; Mohamed, S.; Agarwal, A.; Belgrave, D.; Cho, K.; and Oh, A., eds., \emph{Advances in Neural Information Processing Systems}, volume~35, 37471--37483. Curran Associates, Inc.

\bibitem[{Kelley(1928)}]{Kelley:1928CrossroadsMind}
Kelley, T.~L. 1928.
\newblock \emph{Crossroads in the mind of man; a study of differentiable mental abilities}.
\newblock Palo Alto: Stanford Univ. Press.

\bibitem[{Kleijn and van~der Vaart(2012)}]{Kleijn:2012misspecifiedBvM}
Kleijn, B.; and van~der Vaart, A. 2012.
\newblock {The Bernstein-Von-Mises theorem under misspecification}.
\newblock \emph{Electronic Journal of Statistics}, 6(none): 354 -- 381.

\bibitem[{Koh and Liang(2017)}]{Koh:2017UnderstandingBlackBox}
Koh, P.~W.; and Liang, P. 2017.
\newblock Understanding Black-box Predictions via Influence Functions.
\newblock In Precup, D.; and Teh, Y.~W., eds., \emph{Proceedings of the 34th International Conference on Machine Learning}, volume~70 of \emph{Proceedings of Machine Learning Research}, 1885--1894. PMLR.

\bibitem[{Lakshminarayanan, Pritzel, and Blundell(2017)}]{Lakshminarayanan2017:ensembles}
Lakshminarayanan, B.; Pritzel, A.; and Blundell, C. 2017.
\newblock Simple and Scalable Predictive Uncertainty Estimation using Deep Ensembles.
\newblock In Guyon, I.; Luxburg, U.~V.; Bengio, S.; Wallach, H.; Fergus, R.; Vishwanathan, S.; and Garnett, R., eds., \emph{Advances in Neural Information Processing Systems}, volume~30. Curran Associates, Inc.

\bibitem[{Lam et~al.(2023)Lam, Sanchez-Gonzalez, Willson, Wirnsberger, Fortunato, Alet, Ravuri, Ewalds, Eaton-Rosen, Hu, Merose, Hoyer, Holland, Vinyals, Stott, Pritzel, Mohamed, and Battaglia}]{lam:2023GraphCast}
Lam, R.; Sanchez-Gonzalez, A.; Willson, M.; Wirnsberger, P.; Fortunato, M.; Alet, F.; Ravuri, S.; Ewalds, T.; Eaton-Rosen, Z.; Hu, W.; Merose, A.; Hoyer, S.; Holland, G.; Vinyals, O.; Stott, J.; Pritzel, A.; Mohamed, S.; and Battaglia, P. 2023.
\newblock Learning skillful medium-range global weather forecasting.
\newblock \emph{Science}, 382(6677): 1416--1421.

\bibitem[{Lambert(1765)}]{Lambert:1765Beytraege}
Lambert, J.~H. 1765.
\newblock \emph{Beytr{\"a}ge zum Gebrauche der Mathematik und deren Anwendung}, volume 1, chapter 13 of \emph{Beytr{\"a}ge zum Gebrauche der Mathematik und deren Anwendung}.
\newblock Verlag des Buchladens der Realschule.

\bibitem[{Laplace(1774)}]{laplace1774:integral}
Laplace, P.~S. 1774.
\newblock Mémoire sur la probabilité des causes par les événements.
\newblock \emph{Mémoires de Mathématique et de Physique}, 6.

\bibitem[{Le~Cam(1953)}]{lecam1953:bernsteinVonMises}
Le~Cam, L. 1953.
\newblock \emph{On some asymptotic properties of maximum likelihood estimates and related {Baye}'s estimates}.
\newblock University of California Press, Berkeley.

\bibitem[{MacKay(1992{\natexlab{a}})}]{MacKay:92c}
MacKay, D. J.~C. 1992{\natexlab{a}}.
\newblock Information-based objective functions for active data selection.
\newblock \emph{Neural Computation}, 4(2): 550--604.

\bibitem[{MacKay(1992{\natexlab{b}})}]{MacKay:92b}
MacKay, D. J.~C. 1992{\natexlab{b}}.
\newblock A Practical {Bayesian} Framework for Backpropagation Networks.
\newblock \emph{Neural Computation}, 4: 448--472.

\bibitem[{Maddox et~al.(2019)Maddox, Izmailov, Garipov, Vetrov, and Wilson}]{Maddox2019:SWAG}
Maddox, W.~J.; Izmailov, P.; Garipov, T.; Vetrov, D.~P.; and Wilson, A.~G. 2019.
\newblock A Simple Baseline for Bayesian Uncertainty in Deep Learning.
\newblock In Wallach, H.; Larochelle, H.; Beygelzimer, A.; d\textquotesingle Alch\'{e}-Buc, F.; Fox, E.; and Garnett, R., eds., \emph{Advances in Neural Information Processing Systems}, volume~32. Curran Associates, Inc.

\bibitem[{Magnus(1985)}]{Magnus:1985Eigenvalues}
Magnus, J.~R. 1985.
\newblock On Differentiating Eigenvalues and Eigenvectors.
\newblock \emph{Econometric Theory}, 1(2): 179--191.

\bibitem[{Mahalanobis(1936)}]{Mahalanobis:1936Distance}
Mahalanobis, P.~C. 1936.
\newblock On The Generalized Distance in Statistics.
\newblock \emph{Sankhyā: The Indian Journal of Statistics, Series A (2008-)}, 80: pp. S1--S7.

\bibitem[{Martens(2014)}]{Martens:2014}
Martens, J. 2014.
\newblock New perspectives on the natural gradient method.
\newblock \emph{CoRR}, abs/1412.1193.

\bibitem[{Martens and Grosse(2015)}]{Martens:2015}
Martens, J.; and Grosse, R.~B. 2015.
\newblock Optimizing Neural Networks with Kronecker-factored Approximate Curvature.
\newblock In \emph{Proceedings of the 32nd International Conference on Machine Learning}, volume~37, 2408--2417.

\bibitem[{Miller(1974)}]{Miller:1974Jackknife}
Miller, R.~G. 1974.
\newblock The Jackknife--A Review.
\newblock \emph{Biometrika}, 61(1): 1--15.

\bibitem[{Nilsen et~al.(2022)Nilsen, Munthe-Kaas, Skaug, and Brun}]{Nilsen:2022DeltaMethod}
Nilsen, G.~K.; Munthe-Kaas, A.~Z.; Skaug, H.~J.; and Brun, M. 2022.
\newblock Epistemic uncertainty quantification in deep learning classification by the Delta method.
\newblock \emph{Neural Networks}, 145: 164--176.

\bibitem[{Osband et~al.(2023)Osband, Wen, Asghari, Dwaracherla, IBRAHIMI, Lu, and Van~Roy}]{Osband2023:EpistemicNeuralNetworks}
Osband, I.; Wen, Z.; Asghari, S.~M.; Dwaracherla, V.; IBRAHIMI, M.; Lu, X.; and Van~Roy, B. 2023.
\newblock Epistemic Neural Networks.
\newblock In Oh, A.; Naumann, T.; Globerson, A.; Saenko, K.; Hardt, M.; and Levine, S., eds., \emph{Advances in Neural Information Processing Systems}, volume~36, 2795--2823. Curran Associates, Inc.

\bibitem[{Quenouille(1949)}]{Quenouille:1949Correlation}
Quenouille, M.~H. 1949.
\newblock Approximate Tests of Correlation in Time-Series.
\newblock \emph{Journal of the Royal Statistical Society. Series B (Methodological)}, 11(1): 68--84.

\bibitem[{Ritter, Botev, and Barber(2018)}]{ritter2018:laplace}
Ritter, H.; Botev, A.; and Barber, D. 2018.
\newblock A Scalable Laplace Approximation for Neural Networks.
\newblock In \emph{International Conference on Learning Representations}.

\bibitem[{Schnaus et~al.(2023)Schnaus, Lee, Cremers, and Triebel}]{Schnaus:2023LearningPriors}
Schnaus, D.; Lee, J.; Cremers, D.; and Triebel, R. 2023.
\newblock Learning Expressive Priors for Generalization and Uncertainty Estimation in Neural Networks.
\newblock In Krause, A.; Brunskill, E.; Cho, K.; Engelhardt, B.; Sabato, S.; and Scarlett, J., eds., \emph{Proceedings of the 40th International Conference on Machine Learning}, volume 202 of \emph{Proceedings of Machine Learning Research}, 30252--30284. PMLR.

\bibitem[{Sun et~al.(2022)Sun, Ming, Zhu, and Li}]{Sun:2022OOD}
Sun, Y.; Ming, Y.; Zhu, X.; and Li, Y. 2022.
\newblock Out-of-Distribution Detection with Deep Nearest Neighbors.
\newblock In Chaudhuri, K.; Jegelka, S.; Song, L.; Szepesvari, C.; Niu, G.; and Sabato, S., eds., \emph{Proceedings of the 39th International Conference on Machine Learning}, volume 162 of \emph{Proceedings of Machine Learning Research}, 20827--20840. PMLR.

\bibitem[{Tibshirani(1996)}]{Tibshirani:1996ErrorEstimates}
Tibshirani, R. 1996.
\newblock A Comparison of Some Error Estimates for Neural Network Models.
\newblock \emph{Neural Computation}, 8(1): 152--163.

\bibitem[{Tishby, Levin, and Solla(1989)}]{Tishby:1989InferenceProbabilities}
Tishby, N.; Levin, E.; and Solla, S. 1989.
\newblock Consistent inference of probabilities in layered networks: Predictions and generalization.
\newblock In Anon, ed., \emph{IJCNN Int Jt Conf Neural Network}, 403--409. Publ by IEEE.
\newblock IJCNN International Joint Conference on Neural Networks ; Conference date: 18-06-1989 Through 22-06-1989.

\bibitem[{Tukey(1958)}]{Tukey:1958BiasConfidence}
Tukey, J.~W. 1958.
\newblock Bias and confidence in not-quite large samples (abstract).
\newblock \emph{j-ANN-MATH-STAT}, 29(2): 614--614.

\bibitem[{Van~Amersfoort et~al.(2020)Van~Amersfoort, Smith, Teh, and Gal}]{Amersfoort:2020UncertaintyDUQ}
Van~Amersfoort, J.; Smith, L.; Teh, Y.~W.; and Gal, Y. 2020.
\newblock Uncertainty estimation using a single deep deterministic neural network.
\newblock In \emph{Proceedings of the 37th International Conference on Machine Learning}, ICML'20. JMLR.org.

\bibitem[{van~der Vaart(1998)}]{vandervaart1998:asymptotic}
van~der Vaart, A.~W. 1998.
\newblock \emph{Asymptotic Statistics}.
\newblock Cambridge University Press.

\bibitem[{von Mises(1931)}]{vonMises:1931}
von Mises, R. 1931.
\newblock \emph{Wahrscheinlichkeitsrechnung und ihre {Anwendung} in der {Statistik} und theoretischen {Physik}}, volume~1.
\newblock Franz Deuticke.

\bibitem[{Wright(1934)}]{Wright:1934PathCoefficients}
Wright, S. 1934.
\newblock The method of path coefficients.
\newblock \emph{Ann. Math. Stat.}, 5(3): 161--215.

\end{thebibliography}

\appendix
\section*{Appendix}

\section{Further Analysis}

\subsection{Adversarial Data Interpretation}
\label{sec:delta_adverserial_derivation_appendix}

\newcommand{\deltaAdversarial}{\delta}

Sometimes it is of interest to quantify how much a prediction changes if the training dataset is subject to adversarial data injection. 
Intuitively this is connected to epistemic uncertainty: one may argue that predictions are more robust the more certain we are about their parameters and vice versa.
In this section we will observe that this intuition also holds mathematically.
In particular we will observe that:
\begin{enumerate}
    \item The Delta Variance with $\PosteriorCovarianceMatrix=\frac{1}{N}\iHessian$ computes how much a quantity of interest $\derivedQtheta(z)$ changes if an adversarial data-point is injected.
    \item This adversarial interpretation is technically equivalent to the Laplace Posterior approximation (from Section~\ref{sec:delta_bayesian_derivation}) -- even though interestingly both start with different assumptions and objectives. 
\end{enumerate}
At first we need to generalize the notion of adversarial robustness to the general case of $f\neq\derivedQ$:
We consider the hypothetical scenario where an adversarial fine-tuning-like step on $\derivedQtheta$ is included into the regular training of $\ftheta$. We then quantify the worst-case error this may introduce.
The hypothetical worst-case training scenario then includes an additional data-point $z$ with adversarially selected target value $y$. Then $\theta$ is optimized to minimize prediction error of $\fthetax$ on the training data $\Data$ as well as the $\frac{\epsilon}{N}$-weighted $L_2$ loss for $\derivedQtheta$: I.e. the term $\frac{\epsilon}{2N}(\derivedQtheta(z)-y)^2$ is added to the training loss.

We consider two corruptions where an adversary injects a bad data-point at the very input $z$ that we are interested to evaluate $\derivedQtheta(z)$ at. First adding a data-point $z$ with adversarially selected value offset and second adding a data-point $z$ with noisy value:
\begin{definition}
\label{def:delta_adversarial_offset}
We call a data-point $(z, y)$ \emph{$\deltaAdversarial$-adversarial} for $\derivedQtheta(z)$ if its target value $y$ deviates from the current prediction $\derivedQtheta(z)$ by some arbitrarily large value $\deltaAdversarial$.
\end{definition}
\newcommand{\sigmaAdversarial}{\sigma}
\begin{definition}
\label{def:delta_adversarial_noise}
Similarly we call an data-point $(z, y)$ \emph{$\sigmaAdversarial$-noise-adversarial} for $\derivedQtheta(z)$ if its target value $y$ deviates from the current prediction $\derivedQtheta(z)$ by some zero-mean random variable $\deltaAdversarial$ with variance $\sigmaAdversarial^2$.
\end{definition}

\begin{proposition}
\label{prop:delta_adversarial_offset}
Let $\theta^{adv}$ be the parameters trained after including an $\deltaAdversarial$-adversarial data-point (Definition~\ref{def:delta_adversarial_offset}) at $z$ with $\frac{\epsilon}{N}$-weighted $L_2$ loss into the training.
The prediction then
changes from $\derivedQtheta(z)$ to $\derivedQ_{\theta^{adv}}(z)$. Its difference can be described by the Delta Variance with $\PosteriorCovarianceMatrix=\frac{1}{N}\iHessian$: 
$$
| (\derivedQtheta(z)-\derivedQ_{\theta^{adv}}(z))|
=
\epsilon \left| \deltaAdversarial \quadForm[\PosteriorCovarianceMatrix]{\deltaUz}\right| + \Bigoh{\frac{\epsilon^{2}}{N^{2}}}
$$
In particular:
$$
\underbrace{\left(\quadForm[\PosteriorCovarianceMatrix]{\deltaUz}\right)^2}_{\text{Delta Variance}^2}
    =
    \frac{1}{\deltaAdversarial^2} \lim_{\epsilon\to0}\frac{1}{\epsilon^2} (\derivedQtheta(z)-\derivedQ_{\theta^{adv}}(z))^2
$$
\end{proposition}
\begin{proof}
See Section~\ref{sec:delta_proofs}.
\end{proof}
\begin{proposition}
\label{prop:delta_adversarial_noise}
Let $\theta^{noisy}$ be the parameters trained after including a $\sigma$-noise-adversarial data-point  (Definition~\ref{def:delta_adversarial_noise}) at $z$ with $\frac{\epsilon}{N}$-weighted $L_2$ loss into the training. Then the expected error can be described with the Delta Variance with $\PosteriorCovarianceMatrix=\frac{1}{N}\iHessian$: 
$$ 
\Expectation[\deltaAdversarial]{(\derivedQtheta(z)-\derivedQ_{\theta^{noisy}}(z))^2} 
=
\epsilon^2 \sigma^2 \left(\quadForm[\PosteriorCovarianceMatrix]{\deltaUz}\right)^2 + \Bigoh{\frac{\epsilon^{3}}{N^{3}}}
$$
In particular:
\begin{align*}
\underbrace{\left(\quadForm[\PosteriorCovarianceMatrix]{\deltaUz}\right)^2}_{\text{Delta Variance}^2} 
    &= \frac{1}{\sigma^2} \lim_{\epsilon\to0}\frac{1}{\epsilon^2} \Expectation[\deltaAdversarial]{(\derivedQtheta(z)-\derivedQ_{\theta^{noisy}}(z))^2} 
\end{align*}
\end{proposition}
\begin{proof}
See Section~\ref{sec:delta_proofs}.
\end{proof}

\subsection{Out-of-Distribution Interpretation}
\label{sec:delta_OOD_derivation_appendix}
We show that a large Delta Variance of $\derivedQtheta(z)$ implies that its input $z$ is out-of-distribution with respect to the training data. This relates to epistemic uncertainty intuitively: a model is likely to be uncertain about data-points that differ from its training data. The derivation below is based on the Mahalanobis Distance \cite{Mahalanobis:1936Distance} -- a classic metric for out of distribution detection. It accounts for the possibility that $\ftheta\neq\derivedQtheta$ and relies on minimal assumptions only requiring the existence of gradients and that the training of $\ftheta$ has converged.

\paragraph{Generalized Out-Of-Distribution Detection}
We consider the general case of training $\fthetax$ and evaluating a different $\derivedQtheta(z)$ which also permits differently shaped inputs $x$ and $z$.
To this extent we generalize the notion of out-of-distribution to consider data in different spaces by looking at the training steps associated with each data-point.

While we can not train $\derivedQtheta(z)$ at the unknown test point $z$ we can still consider how such a hypothetical learning step would look like. Computing $\gradtheta \derivedQtheta(z)$ tells us the direction of the learning step without its magnitude.
We can then test if this learning direction would be in distribution with the actual learning steps that were done to train $\fthetax$ on the training data $x\in \Data$.

\begin{proposition}
\label{prop:delta_OOD}
The Delta Variance with $\PosteriorCovarianceMatrix=\hat{F}_f^{-1}$ computes the Mahalanobis Distance between $z$ and the training data $\Data$ in gradient space.
$$ d_M^{\nabla}(z, \Data) = \quadForm[\PosteriorCovarianceMatrix]{\deltaUz}$$
for $d_M^{\nabla}(z, \Data)\defeq d_M(\gradtheta \derivedQtheta(z), \{ \gradtheta \log \pfthetax | (x, y) \in \Data \})$
\end{proposition}
\begin{proof}
See Section~\ref{sec:delta_proofs}.
\end{proof}
Hence the Delta Variance $\quadForm[\PosteriorCovarianceMatrix]{\deltaUz}$ is large if and only if the derived quantity $\derivedQtheta(z)$ is evaluated at an out-of-distribution point $z$.

\paragraph{Connection to Epistemic Uncertainty}
One can draw an intuitive connection to epistemic uncertainty: If a hypothetical training step on $\derivedQtheta(z)$ is in-distribution with respect to (i.e. exchangeable by) the actual training steps that were used to learn $\theta$, then one may argue that $\derivedQtheta(z)$ has already been learned well. Hence intuitively the epistemic uncertainty should of $\derivedQtheta(z)$ be low if and only if the Mahalanobis Distance is small.

\paragraph{Mahalanobis Distance}
The Mahalanobis Distance can be interpreted as fitting a multivariate normal distribution to the data (in our case the data are the gradients from the \enquote{learning steps}) and then computing the log-likelihood of the test point $x$. 
\begin{definition}
The \emph{Mahalanobis Distance} between a vector $x$ and a set of vectors $x_i\in\Data$ is $$ d_M(x, \Data) = \trans{(x-\mu)} C^{-1}(x-\mu)$$ where $C$ is the empirical covariance matrix of $x_i\in \Data$ and $\mu$ the empirical mean. 
\end{definition}

\section{Epistemic Variance of Eigenvalues}
To illustrate the Epistemic Variance computation with the implicit function theorem, we consider the eigenvalue problem and compute the uncertainty of an eigenvalue. 

\paragraph{Eigenvalue Derivation}
Recall that the eigenvalues $\lambda_i$ of matrix $\EigenvalueExampleMatrix$ are the solutions to the characteristic polynomial
$ \det (\EigenvalueExampleMatrix - \lambda \identity )=0$, which is typically solved using iterative algorithms. If some entries of $\EigenvalueExampleMatrix$ are uncertain -- or more generally if $\EigenvalueExampleMatrix_\theta$ is a function of learned parameters $\theta$ -- then its eigenvalues will be a non-trivial function of $\theta$.

We can estimate the Epistemic Variance  $\Variance[\theta]{\lambda_i(\EigenvalueExampleMatrix_\theta)}$ using Delta Variances if we obtain the gradient of the eigenvalue $\gradtheta \lambda_i(\EigenvalueExampleMatrix_\theta)$. To this extend we need the implicit function theorem as the function $\lambda_i(\EigenvalueExampleMatrix_\theta)$ admits no closed form. We refer to the derivation in \citet{Magnus:1985Eigenvalues} which applies to any eigenvalue of multiplicity one:
$\gradtheta \lambda_i(\theta) = \trans{e_i} (\gradtheta \EigenvalueExampleMatrix_\theta) \hat{e}_i \frac{1}{\trans{e_i}\hat{e}_i}$
where $e_i$ and $\hat{e}_i$ are the unit-norm left and right eigenvectors corresponding to $\lambda_i$ of $\EigenvalueExampleMatrix_\theta$ evaluated at the learned parameters $\thetabar$. For auto-diff convenience we can also write $\gradtheta \lambda_i(\theta) = \gradtheta \left( \trans{e_i} \EigenvalueExampleMatrix_\theta \hat{e}_i \frac{1}{\trans{e_i}\hat{e}_i} \right)$ because $e_i$ and $\hat{e}_i$ only depend on $\thetabar$ not on $\theta$, so this serves like a stop-gradient and recovers the formula from \citet{Magnus:1985Eigenvalues}.

\paragraph{Eigenvalue Example}
Eigenvalues are a quantity of interest in structural engineering. There eigenvalues from a finite element model matrix $\EigenvalueExampleMatrix=M_\theta^{-1}K_\theta$ indicate the stability of a physical structure. In this model $M_\theta$ is a mass-matrix and $K_\theta$ is a stiffness matrix. Both are sparse matrices with entries determined by mass and stiffness parameters $\theta$. In Figure~\ref{fig:delta_finite_element_eigenvalue} we consider an illustrative example with 5 elements of weight $1kg$ that are connected with springs to each other and the boundary of increasing stiffness ranging from 1 to 6 $N/m$. The set of 5 weight and 6 stiffness constants are the parameters $\theta$.
Figure~\ref{fig:delta_finite_element_eigenvalue} assumes a posterior with variance $10^{-2}$ around each parameter in $\theta$ and compares the actual distribution for each eigenvalue with its Delta Variance using the gradient from above.

\section{Experimental Methodology}

\paragraph{Variance Prediction Methodology}
To obtain the Bootstrapped Ensembles variance we train $10$ separate neural networks $\ftheta$ with different initial parameters \cite{Lakshminarayanan2017:ensembles}. Similarly we compute the Monte-Carlo Dropout variance by evaluating the neural network $10$ times with different dropout samples \cite{Gal2016:dropout}.
As the graph neural network in \citet{lam:2023GraphCast} does not employ Monte-Carlo Dropout during training, we introduce it post-hoc for evaluation only. We also consider a convolutional neural network that we train with dropout using the same training procedure. 
Delta Variances are implemented with a diagonal Fisher information matrix computed using EMA on training batches sized $32$ with a decay of $10^{-3}$. This value is chosen such that the expected window of $32\times10^{3}$ roughly matches the number of data-points 
$\approx 5\times10^4$.
We use the interval from 2014 to 2017 to select hyper-parameters for each method and each $\derivedQtheta$. For dropout we consider 14 ratios between $5*10^{-3}$ and $0.8$ for Delta Variances we consider a regularisation in powers of $10$ between $10^{-15}$ and $10^9$. As described in Section~\ref{sec:delta_learning_sigma} Fine-tuned Delta Variances fine-tune on this data. The data from 2018-2021 is held out for evaluation.

\section{Proofs}
\label{sec:delta_proofs}
\newcommand{\thetalooepsilon}{\theta^{\epsilon}_i}

\paragraph{Proof of Proposition~\ref{prop:delta_bound}:}
\begin{proposition*}
For a normally distributed posterior with mean $\thetabar$ and a covariance matrix $\PosteriorCovarianceMatrix$ proportional to $\frac{1}{N}$ 
it holds:
$$ \underbrace{\Variance[\theta \sim p(\theta|f, \Data)]{\derivedQtheta(z)}}_{\text{Epistemic Variance}} = \underbrace{\quadForm[\PosteriorCovarianceMatrix]{\deltaUz}}_{\text{Delta Variance}} + \Bigoh{N^{-1.5}}$$
where $\deltaUz\defeq \gradtheta \derivedQtheta(z)|_{\theta=\thetabar}$ as usual.
\end{proposition*}
\begin{proof}
This is an application of the Delta Method~\cite{Lambert:1765Beytraege,Doob:1935LimitingDistributions}, which is typically stated and proved a bit differently and without error term. We adapt it slightly to better match our framework.
We begin by noting two helpful facts: Since $\PosteriorCovarianceMatrix$ is proportional to $\frac{1}{N}$ we can write $\PosteriorCovarianceMatrix=\frac{1}{N}C$ for some matrix $C$. Hence sampling from the normal posterior is achieved by
$\theta = \thetabar + \epsilon * \deltaThetaSample$ with $\epsilon^2\defeq\frac{1}{N}$ and $ \deltaThetaSample \sim \normal{0, C}$.
Furthermore the variance of the dot-product between any vector $v$ and a multi-variate normal variable is given by $\Variance[z \sim \normal{0, C}]{\trans{z} v}=\quadForm[C]{v}$. 
Then we can perform a change of variable \eqref{eq:delta_proof_change_of_variable}, the Taylor expansion \eqref{eq:delta_proof_taylor_expansion} and the dot-product lemma: \eqref{eq:delta_proof_variance_dotproduct}
    \begin{align}
    \Variance[\theta]{\derivedQtheta(x)} 
        &= \Variance[\deltaThetaSample]{\derivedQ_{\thetabar + \epsilon z}(x)} \label{eq:delta_proof_change_of_variable}\\
        &= \Variance[\deltaThetaSample]{\derivedQ_\thetabar(x) + \epsilon \trans{\deltaThetaSample} \nabla\derivedQ_\thetabar(x) + \Bigoh{\epsilon^2}} \label{eq:delta_proof_taylor_expansion}\\
        &= \Variance[\deltaThetaSample]{\epsilon \trans{\deltaThetaSample} \nabla\derivedQ_\thetabar(x) + \Bigoh{\epsilon^2}} \nonumber\\
        &= \epsilon^2 \Variance[\deltaThetaSample]{ \trans{\deltaThetaSample} \nabla\derivedQ_\thetabar(x) + \Bigoh{\epsilon}} \nonumber \\
        &= \epsilon^2 \Variance[\deltaThetaSample]{ \trans{\deltaThetaSample} \nabla\derivedQ_\thetabar(x)} + \Bigoh{\epsilon^3} \label{eq:delta_proof_variance_dotproduct}\\
        &= \frac{1}{N}\quadForm[C]{\nabla \derivedQtheta(x)} + \Bigoh{N^{-1.5}} \nonumber \\
        &= \quadForm[\PosteriorCovarianceMatrix]{\deltaUx} + \Bigoh{N^{-1.5}} \nonumber
    \end{align}

\end{proof}

\paragraph{Lemma regarding influence functions:}
\newcommand{\ExtraLoss}{L'}
Influence functions are typically used to compute the infinitesimal change introduced by down weighting a training point. Later we will also use them for introducing a novel data-point with infinitesimal weight.
This infinitesimal change to the training objective changes its maximum and optimal parameters.
We follow the derivations in \citet{cook:1982residuals} and \citet{Koh:2017UnderstandingBlackBox} to approximate the new parameters $\theta^{*}$ that maximize the new objective with the difference. We extend said derivations in order to keep track of the error terms explicitly.
\begin{lemma}
\label{lemma:delta_influence_new_parameters}
Adding an objective $L(\theta)$ with optimal parameters $\thetabar$ and an infinitesimally $\frac{\epsilon}{N}$-weighted objective $\ExtraLoss(\theta)$ results in new optimal parameters of the combined objective
$$ \theta^* = \arg\min_\theta L(\theta) + \frac{\epsilon}{N} \ExtraLoss(\theta)$$
Assuming that the Hessian $H$ of $L(\theta)$ at $\thetabar$ is invertible the new optimum is approximated by
\begin{equation}
    \theta^* = \thetabar - \frac{\epsilon}{N} H^{-1}  \gradtheta \ExtraLoss(\thetabar) + \bigoh\left(\frac{\epsilon^2}{N^2}\right) \label{eq:delta_influence_new_parameters}
\end{equation}
\end{lemma}
\begin{proof}
The new optimum $\theta^{*}$ of $L(\theta) + \frac{\epsilon}{N} \ExtraLoss(\theta)$ is attained at
$$ \gradtheta L(\theta) + \frac{\epsilon}{N} \gradtheta \ExtraLoss(\theta) = 0 $$
We perform a Taylor expansion at the optimum $\thetabar$ of $L(\theta)$ 
\begin{align*}
\gradtheta L(\theta) &+ \frac{\epsilon}{N} \gradtheta \ExtraLoss(\theta) 
    =  
    \gradtheta L(\thetabar) + \frac{\epsilon}{N} \gradtheta \ExtraLoss(\thetabar) - \\
    & \left(H + \frac{\epsilon}{N} \gradtheta^2 \ExtraLoss(\thetabar)  \right)(\theta-\thetabar) + \bigoh(||(\theta-\thetabar)||^2)
\end{align*}
Equating this to zero, using $\gradtheta L(\thetabar)=0$ and that $ (A + \epsilon B)^{-1} = A^{-1} + \bigoh(\epsilon)$ for matrices $A, B$ and sufficiently small $\epsilon$ we obtain the new optimum:
\begin{align*}
\theta^*-\thetabar 
    &= -\left(H^{-1} + \bigoh\left(\frac{\epsilon}{N}\right)\right) \left(\frac{\epsilon}{N} \gradtheta \ExtraLoss(\thetabar) + \bigoh(||(\theta^*-\thetabar)||^2) \right) \\
    &= -\frac{\epsilon}{N} H^{-1} \gradtheta \ExtraLoss(\thetabar) + \bigoh\left(\frac{\epsilon^2}{N^2}\right) + \bigoh(||(\theta^*-\thetabar)||^2) \\
    &= -\frac{\epsilon}{N} H^{-1} \gradtheta \ExtraLoss(\thetabar) + \bigoh\left(\frac{\epsilon^2}{N^2}\right)
\end{align*}
Where we used Lemma~\ref{lemma:delta_lemma_joint_convergence} with $x\defeq\theta^*-\thetabar$ and $a\defeq -H^{-1} \gradtheta \ExtraLoss(\thetabar)$ in the final step.
\end{proof}

\paragraph{Joint Convergence Lemma}
\begin{lemma}
\label{lemma:delta_lemma_joint_convergence}
Let $x$ and $a$ be vectors and $\epsilon$ be a scalar satisfying
$$ x + \bigoh(||x||^2)= \epsilon a + \bigoh(\epsilon^2) $$
when $\epsilon\to0$ the following holds 
$$x = \epsilon a + \bigoh(\epsilon^2)$$
\end{lemma}
\begin{proof}
Then from $\bigoh(||x||^2) = \epsilon a + \bigoh(\epsilon^2) - x$ it follows
\begin{align}
\bigoh(||x||) &= \frac{1}{||x||}||\epsilon a + \bigoh(\epsilon^2) - x|| \nonumber \\
&= || \frac{\epsilon a + \bigoh(\epsilon^2)}{||x||} - \frac{x}{||x||} || \label{eq:delta_lemma_joint_convergence_rhs}
\end{align}
Observe that $\epsilon\to0$ implies $||x||\to0$. Hence $\lim_{\epsilon\to0}\bigoh(||x||)=0$ which in turn implies that \eqref{eq:delta_lemma_joint_convergence_rhs} goes to zero as $||x||\to0$. Observe that $\frac{x}{||x||}$ is a unit-vector for all $x$. Hence \eqref{eq:delta_lemma_joint_convergence_rhs} can only become zero if the numerators converge:
$$ \lim_{\epsilon\to0} \left( \epsilon a + \bigoh(\epsilon^2) - x \right)= 0$$
Consequently $x$ is asymptotically linear in $\epsilon$ thus $x = \bigoh(\epsilon)$ and $\bigoh(||x||^2) = \bigoh(\epsilon^2)$ which concludes the proof.
\end{proof}

\paragraph{Proof of Proposition~\ref{prop:delta_infinitessimal}:}
\begin{proposition*}
The Delta Variance equals the infinitesimal LOO Variance for $\PosteriorCovarianceMatrix=\frac{1}{N}\iHessian \hat{F}_f \iHessian$:
$$ \Variance[IJ]{\derivedQtheta(z)} = \quadForm[\PosteriorCovarianceMatrix]{\deltaUz} $$
In particular
$$ \Variance[\theta \sim \epsilon-LOO]{\derivedQtheta(z)} = \quadForm[\PosteriorCovarianceMatrix]{\deltaUz} + \bigoh\left(\frac{\epsilon}{N^2}\right)  $$
\end{proposition*}

\begin{proof}
Let $\thetalooepsilon$ be the parameters obtained after retraining $\theta$ with a single data-point $x_i$ down-weighted from weight $\frac{1}{N}$ to $\frac{1-\epsilon}{N}$ and let $\thetabar$ be the original parameters obtained from training on all data with weight $\frac{1}{N}$. Then from Lemma~\ref{lemma:delta_influence_new_parameters} we obtain for $\frac{\epsilon}{N}\to0$
$$ \thetalooepsilon = \thetabar + \frac{\epsilon}{N} \iHessian \gradtheta \log \pfthetaxi + \bigoh\left(\frac{\epsilon^2}{N^2}\right) $$
Now define the parameter difference as
$$ \deltatheta \defeq \thetalooepsilon - \thetabar = \frac{\epsilon}{N} \iHessian \gradtheta \log \pfthetaxi + \bigoh\left(\frac{\epsilon^2}{N^2}\right) $$
and apply Taylors Formula:
\begin{align*}
\derivedQ_{\thetalooepsilon}(z) 
    &= \derivedQthetabar(z) + \gradtheta \trans{\derivedQthetabar(z)} \deltatheta + \bigoh(||\deltatheta||^2) \\
    &= \derivedQthetabar(z) + \frac{\epsilon}{N} \gradtheta \trans{\derivedQthetabar(z)}  \iHessian \gradtheta \log \pfthetaxi \\&+ \bigoh\left(\frac{\epsilon^2}{N^2}\right)
\end{align*}
Hence
\begin{align}
&\Variance[\theta \sim \epsilon-LOO]{\derivedQtheta(z)} 
\defeq \frac{N}{\epsilon^2}\Variance[\iinN]{\derivedQ_{\thetalooepsilon}(z)} \nonumber \\
    &= \frac{1}{N} \Variance[i]{
     \gradtheta \trans{\derivedQthetabar(z)}  \iHessian \gradtheta \log \pfthetaxi + \bigoh\left(\frac{\epsilon}{N}\right) 
     } \nonumber\\
    &= \frac{1}{N} \Variance[i]{
     \gradtheta \trans{\derivedQthetabar(z)}  \iHessian \gradtheta \log \pfthetaxi 
    } \nonumber
    \\ &+ \bigoh\left(\frac{\epsilon}{N^2}\right) \label{eq:delta_proof_loo_zero_expectation}\\
    &= \frac{1}{N} \Expectation[i]{
     \left(
     \gradtheta \trans{\derivedQthetabar(z)}  \iHessian \gradtheta \log \pfthetaxi
     \right)^2
    } \nonumber \\
    &+ \bigoh\left(\frac{\epsilon}{N^2}\right) \label{eq:delta_proof_loo_gradient_expectation} \\
    &= \frac{1}{N}
    \quadForm[\iHessian \hat{F}_f \iHessian]{\deltaUz}  + \bigoh\left(\frac{\epsilon}{N^2}\right) \nonumber
\end{align}
In step \ref{eq:delta_proof_loo_zero_expectation} we use that $\theta$ is a at local optimum of the loss such that $\Expectation[i]{\gradtheta\log\pfthetaxi}=0$ and in step \ref{eq:delta_proof_loo_gradient_expectation} we use that the expectation is only with respect to the $\ftheta(x_i)$ terms and $\hat{F}_f=\Expectation[i]{\quadForm[]{\gradtheta \log \pfthetaxi}}$.
\end{proof}

\paragraph{Proof of Proposition~\ref{prop:delta_boostrapping}:}
\begin{proposition*}
The Delta Variance equals the $\epsilon$-Bootstrapped Variance with a diminishing approximation error for $\PosteriorCovarianceMatrix=\frac{1}{N} \iHessian \hat{F}_f \iHessian$:
$$ \Variance[\theta \sim \epsilon-Bootstrap]{\derivedQtheta(z)} = \quadForm[\PosteriorCovarianceMatrix]{\deltaUz} + \bigoh\left( \frac{\epsilon}{N^{2.5}} \right)$$
\end{proposition*}
\begin{proof}
    The proof mirrors the proof for Proposition~\ref{prop:delta_infinitessimal} with a different $\deltatheta$ obtained from Lemma~\ref{lemma:delta_bootstrapping_new_parameters}: $\thetaRademacher = \thetabar + \deltatheta$ with $$\deltatheta = \frac{\epsilon}{N} \iHessian G \bootRademacher + \bigoh\left(\frac{\epsilon^2}{N^{1.5}}\right)$$
    Again we apply Taylors Formula:
\begin{align*}
\derivedQ_{\thetaRademacher}(z) 
    &= \derivedQthetabar(z) + \gradtheta \trans{\derivedQthetabar(z)} \deltatheta + \bigoh(||\deltatheta||^2) \\
    &= \derivedQthetabar(z) + \frac{\epsilon}{N} \gradtheta \trans{\derivedQthetabar(z)}  \iHessian \gradtheta G \bootRademacher + \bigoh\left(\frac{\epsilon^2}{N^{1.5}}\right)
\end{align*}
Next we use that $\Expectation{G \bootRademacher \trans{\bootRademacher} \trans{G}} = \Expectation{G \trans{G}} = \hat{F}_f N$
\begin{align}
&\Variance[\theta \sim \epsilon-Bootstrap]{\derivedQtheta(z)} 
\defeq \frac{1}{\epsilon^2}\Variance[\bootRademacher]{\derivedQ_{\thetaRademacher}(z)} \nonumber \\
    &= \frac{1}{\epsilon^2}\Variance[\bootRademacher]{
    \derivedQthetabar(z) + \frac{\epsilon}{N} \gradtheta \trans{\derivedQthetabar(z)}  \iHessian  G \bootRademacher + \bigoh\left(\frac{\epsilon^2}{N^{1.5}}\right)
    } \nonumber \\
    &= \frac{1}{N^2} \Variance[\bootRademacher]{
     \gradtheta \trans{\derivedQthetabar(z)}  \iHessian  G \bootRademacher + \bigoh\left(\frac{\epsilon}{N^{0.5}}\right)
    } \nonumber \\
    &= \frac{1}{N^2} \Variance[\bootRademacher]{
     \gradtheta \trans{\derivedQthetabar(z)}  \iHessian G \bootRademacher
    } + \bigoh\left(\frac{\epsilon}{N^{2.5}}\right) \nonumber \\
    &= \frac{1}{N^2} \Expectation[\bootRademacher]{
     \gradtheta \trans{\derivedQthetabar(z)}  \iHessian G \bootRademacher \trans{\bootRademacher} \trans{G} \iHessian \gradtheta \derivedQthetabar(z)
    } + \bigoh\left(\frac{\epsilon}{N^{2.5}}\right) \nonumber \\
    &= \frac{1}{N^2} \Expectation[\bootRademacher]{
     \gradtheta \trans{\derivedQthetabar(z)}   \iHessian N \hat{F}_f \iHessian \gradtheta \derivedQthetabar(z)
    } + \bigoh\left(\frac{\epsilon}{N^{2.5}}\right) \nonumber \\
    &=  \Expectation[\bootRademacher]{
     \gradtheta \trans{\derivedQthetabar(z)}   \frac{1}{N} \iHessian \hat{F}_f \iHessian \gradtheta \derivedQthetabar(z)
    } + \bigoh\left(\frac{\epsilon}{N^{2.5}}\right) \nonumber \\
    &= 
    \DVariance{\derivedQtheta(z)}
     + \bigoh\left(\frac{\epsilon}{N^{2.5}}\right) \nonumber 
\end{align}
with $\PosteriorCovarianceMatrix=\frac{1}{N} \iHessian \hat{F}_f \iHessian$.
\end{proof}

\begin{lemma}
\label{lemma:delta_bootstrapping_new_parameters}
    Let random variable $\theta_r$ be the parameters obtained from training $\ftheta$ on data $\Data$ with each data-point $x_i\in \Data$ re-weighted randomly to either $\frac{1-\epsilon}{N}$ or $\frac{1+\epsilon}{N}$. Let vector $r$ (where all $r_i$ are Rademacher distributed) determine the sign: $\frac{1+r_i\epsilon}{N}$,
    then
    $\theta_r = \thetabar + \deltatheta$ with $$\deltatheta = \frac{\epsilon}{N} \iHessian G r + \bigoh\left(\frac{\epsilon^2}{N^{1.5}}\right)$$
    for the $i$-th column in $G$ being the training gradient for data-point $x_i$: $\gradtheta\log \pfthetaxi$.
    Furthermore $\Expectation[r]{\deltatheta}=\bigoh\left(\frac{\epsilon^2}{N^{1.5}}\right)$ and $\Expectation[r]{||\deltatheta||_2^2}=\bigoh\left(\frac{\epsilon^2}{N}\right)$.
\end{lemma}
\begin{proof}
The proof is identical to Lemma~\ref{lemma:delta_influence_new_parameters}, with the exception that $\ExtraLoss(\thetabar)$ depends on $N$. In particular $\ExtraLoss(\thetabar)$ being the summation of $N$ Rademacher weighted losses makes $\ExtraLoss(\thetabar)$ grow probabilistically with order $\sqrt{N}$. 
Hence also $\gradtheta^2 \ExtraLoss(\thetabar) = \bigoh(\sqrt{N})$.
Adapting the proofs of Lemma~\ref{lemma:delta_influence_new_parameters} and \ref{lemma:delta_lemma_joint_convergence} we obtain:
\begin{align*}
\theta^*-\thetabar 
    &= -\left(H^{-1} + \bigoh\left(\frac{\epsilon}{N^{0.5}}\right)\right) \left(\frac{\epsilon}{N} \gradtheta \ExtraLoss(\thetabar) + \bigoh(||(\theta^*-\thetabar)||^2) \right) \\
    &= -\frac{\epsilon}{N} H^{-1} \gradtheta \ExtraLoss(\thetabar) + \bigoh\left(\frac{\epsilon^2}{N^{1.5}}\right) + \bigoh(||(\theta^*-\thetabar)||^2) \\
    &= -\frac{\epsilon}{N} H^{-1} \gradtheta \ExtraLoss(\thetabar) + \bigoh\left(\frac{\epsilon^2}{N^{1.5}}\right)
\end{align*}
Note that $\Expectation[r]{\deltatheta}=\bigoh\left(\frac{\epsilon^2}{N^{1.5}}\right)$ because $\Expectation[r]{Gr}=0$ follows from $\Expectation[r]{r}=0$. Finally recall that $\Expectation[r]{G \bootRademacher \trans{\bootRademacher} \trans{G}} = \hat{F}_f N$ implies $\Expectation[r]{|| H^{-1} G r ||_2^2} = \bigoh(N)$, hence
$\Expectation[r]{||\deltatheta||_2^2}=\Bigoh{\Expectation[r]{||\frac{\epsilon}{N} H^{-1} \gradtheta \ExtraLoss(\thetabar)||_2^2}} = \Bigoh{\frac{\epsilon^2}{N^2} \Expectation[r]{|| H^{-1} G r ||_2^2}} = \Bigoh{\frac{\epsilon^2}{N}}$.
\end{proof}

\paragraph{Proof of Propositions~\ref{prop:delta_adversarial_offset} and~\ref{prop:delta_adversarial_noise}:}
\begin{proof}
\newcommand{\thetaadverserial}{\theta^{\epsilon-adv}_{z,y}}
Using the technique of influence functions -- see Lemma~\ref{lemma:delta_influence_new_parameters} -- we can compute the parameters $\theta^{bad}$ (i.e. $\theta^{adv}$ and $\theta^{noisy}$) after re-training with the adversarial data-point $(z, y)$ included with $L_2$-loss and $\frac{\epsilon}{N}$-weight. Let $\thetabar$ be the optimal parameters before re-training, then
\begin{align*}
\deltatheta 
    &\defeq \theta^{bad} - \thetabar \\
    &= \frac{\epsilon}{N} \iHessian \gradtheta \frac{1}{2}(\derivedQtheta(z)-y)^2 +  \Bigoh{\frac{\epsilon^2}{N^2}} \\
    &= \frac{\epsilon}{N} \iHessian \deltaAdversarial \gradtheta \derivedQtheta(z) +  \Bigoh{\frac{\epsilon^2}{N^2}}
\end{align*}
Where $\deltaAdversarial$ is the prediction error $\deltaAdversarial\defeq\derivedQtheta(z)-y$.
Similarly to Proposition~\ref{prop:delta_infinitessimal} we perform a Taylor approximation around $\thetabar$ and plug $\deltatheta$ in:
\begin{align*}
\derivedQthetabar(z)- \derivedQ_{\theta^{bad}}(z)
    &= \frac{\epsilon}{N} \gradtheta \trans{\derivedQthetabar(z)}  \iHessian \deltaAdversarial \gradtheta \derivedQtheta(z) + \Bigoh{\frac{\epsilon^2}{N^2}} \\
    &= \frac{\epsilon}{N} \deltaAdversarial \quadForm[\iHessian]{\deltaUz} + \Bigoh{\frac{\epsilon^2}{N^2}}
\end{align*}
hence 
$$
|\derivedQthetabar(z)- \derivedQ_{\theta^{bad}}(z)| = \frac{\epsilon}{N} \left| \deltaAdversarial \quadForm[\iHessian]{\deltaUz}\right| + \Bigoh{\frac{\epsilon^2}{N^2}}
$$
This concludes the proof for Proposition~\ref{prop:delta_adversarial_offset}, so we proceed to Proposition~\ref{prop:delta_adversarial_noise} where $\deltaAdversarial$ is a zero-mean random variable with variance $\sigma^2$. Using short hand notation $\nu_z \defeq \quadForm[\iHessian]{\deltaUz}$
\begin{align*}
\Expectation[\deltaAdversarial]{ (\derivedQthetabar(z)-\derivedQ_{\theta^{noisy}}(z))^2 }
    &= \Expectation[\deltaAdversarial]{\left(\frac{\epsilon}{N} \deltaAdversarial \nu_z + \Bigoh{\frac{\epsilon^2}{N^2}}\right)^2}
    \\
    &= \Expectation[\deltaAdversarial]{\frac{\epsilon^2}{N^2} \deltaAdversarial^2 \nu_z^2 + \Bigoh{\frac{\epsilon^3}{N^3}} }
    \\
    &= \frac{\epsilon^2}{N^2}\nu_z^2 \Expectation[\deltaAdversarial]{\deltaAdversarial^2}  + \Bigoh{\frac{\epsilon^3}{N^3}}
    \\
    &= \frac{\epsilon^2}{N^2} \sigma^2 \nu_z^2 + \Bigoh{\frac{\epsilon^3}{N^3}} \\
    &= \epsilon^2 \sigma^2  \left(\quadForm[\frac{1}{N}\iHessian]{\deltaUz}\right)^2 
    \\ &+ \Bigoh{\frac{\epsilon^3}{N^3}}
\end{align*}
The limits follow trivially.
\end{proof}

\paragraph{Proof of Propositions~\ref{prop:delta_OOD}:}

\begin{proof}
The proof relies on two insights: Firstly $\mu=0$ because by assumption we have trained until convergence such that 
$\sum_{(x, y)\in \Data}\gradtheta\log\pfthetax=0$.
Secondly by definition the previous fact:
$$C\defeq\frac{1}{N}\sum_{(x, y)\in\Data}\quadForm[]{\gradtheta\log\pfthetax} \defeq \hat{F}_f$$
Thus
\begin{align*}
d_M^{\phi}(z, \Data) 
    &= \quadForm[C^{-1}]{\gradtheta\derivedQtheta(z)} \\
    &= \quadForm[\hat{F}_f^{-1}]{\deltaUz}
\end{align*}
\end{proof}

\nocite{vonMises:1931}
\nocite{Bernstein:TheoryProbabilities}
\nocite{lecam1953:bernsteinVonMises}
\nocite{Kleijn:2012misspecifiedBvM}
\nocite{Gorroochurn:2020WhoInventedDelta}
\nocite{cotes:1722Harmonia}
\nocite{gauss:1823ErrorTheory}
\nocite{Lambert:1765Beytraege}
\nocite{Kelley:1928CrossroadsMind,Wright:1934PathCoefficients, Doob:1935LimitingDistributions,Dorfman:1938note}
\nocite{Denker:1990Transforming}
\nocite{Tibshirani:1996ErrorEstimates}
\nocite{Hwang:1997PredictionIntervals}
\nocite{Veaux:1998PredictionIntervals}
\nocite{Nilsen:2022DeltaMethod}
\nocite{Gull:1989MaximumEntropy}
\nocite{Tishby:1989InferenceProbabilities}
\nocite{MacKay:92b}
\nocite{ritter2018:laplace}
\nocite{Martens:2015}
\nocite{Immer:2021LocalLinearization}
\nocite{Miller:1974Jackknife} 
\nocite{Tukey:1958BiasConfidence}
\nocite{Quenouille:1949Correlation}
\nocite{Jaeckel:1972InfinitesimalJackknife}
\nocite{Hampel:1974Influence}
\nocite{Hodges1967:Efficiency}
\nocite{cook:1982residuals}
\nocite{Efron:1979}
\nocite{Bickel1981:Bootstrap}
\nocite{Koh:2017UnderstandingBlackBox}

\end{document}